\documentclass[a4paper,11pt,openany]{article}
\usepackage[utf8]{inputenc}
\usepackage{amsmath,amssymb,amsfonts}
\usepackage{amsthm}
\usepackage{geometry}
\usepackage[ruled]{algorithm2e}
\usepackage{graphicx}
\usepackage[english]{babel}
\usepackage{listings}
\usepackage[toc,page]{appendix}
\usepackage{supertabular}
\usepackage{longtable}
\usepackage{enumerate}
\usepackage[table]{xcolor}
\usepackage{array} 
\usepackage{float}
\usepackage{caption}
\usepackage{booktabs}
\usepackage{natbib}
\usepackage{multirow}
\setcitestyle{authoryear}
\usepackage{mathtools}

\usepackage{blindtext}
 
\usepackage{layaureo}

\graphicspath{{./Images/}}
\usepackage{url,hyperref}
\usepackage{color}
\usepackage{comment}
\usepackage{mathrsfs}  

\newtheorem{thm}{Theorem}

\usepackage[normalem]{ulem}
\newcounter{ToDo}
\newcounter{gaocomm}
\newcounter{Note}
\definecolor{blue-violet}{rgb}{0.54, 0.17, 0.89}
\definecolor{mygreen}{rgb}{0.0, 0.5, 0.0}
\definecolor{awesome}{rgb}{1.0, 0.13, 0.32}
\definecolor{bostonuniversityred}{rgb}{0.8, 0.0, 0.0}

\usepackage{authblk}
\title{Adaptive Multi-level Hyper-gradient Descent}
\author[1]{Renlong Jie}
\author[1]{Junbin Gao}
\author[1]{Andrey Vasnev}
\author[1]{Minh-Ngoc Tran}
\affil[1]{The University of Sydney\\
Camperdown NSW 2006}
\title{Adaptive Multi-level Hyper-gradient Descent}

\begin{document}

\maketitle

\begin{abstract}
In this study, we investigate learning rate adaption at different levels based on the hyper-gradient descent framework and propose a method that adaptively learns the optimizer parameters by combining multiple levels of learning rates with hierarchical structures. Meanwhile, we show the relationship between regularizing over-parameterized learning rates and building combinations of adaptive learning rates at different levels. The experiments on several network architectures, including feed-forward networks, LeNet-5 and ResNet-18/34, show that the proposed multi-level adaptive approach can outperform baseline adaptive methods in a variety of circumstances.
\end{abstract}

\section{Introduction}\label{Sec:5.1}

The basic optimization algorithm for training deep neural networks is gradient descent method (GD),
including stochastic gradient descent (SGD), mini-batch gradient descent and batch gradient descent. Model parameters are updated according to the first-order gradients of the empirical risks with respect to the parameters being optimized, while back-propagation is implemented for calculating the gradients of parameters \citep{ruder2016overview}.
Na\"{i}ve gradient descent methods apply fixed learning rates without any adaptation mechanisms. However, considering the change of available information during the learning process, SGD with fixed learning rates can result in inefficiency and a waste of computing resources in hyper-parameter searching. One solution is to introduce adaptive updating rules, while the learning rates are still fixed in training. This leads to the proposed methods include AdamGrad \citep{duchi2011adaptive}, RMSProp \citep{tieleman2012rmsprop}, and Adam \citep{kingma2014adam}. Also there are optimizers aiming at addressing the convergence issue in Adam \citep{reddi2019convergence, luo2019adaptive}, or rectify the variance of the adaptive learning rate \citep{liu2019variance}. Other techniques such as lookahead could also achieve variance reduction and stability improvement with negligible extra computational cost \citep{zhang2019lookahead}.
\newline\newline
Even though the adaptive optimizers with fixed learning rates can converge faster than SGD in a wide range of tasks, the updating rules are designed manually, while more hyper-parameters are introduced. Another idea is to use the information of objective function and to update the learning rates as trainable parameters. This set of methods was introduced as automatic differentiation, where the hyper-paremeters can be optimized with backpropagation \citep{maclaurin2015gradient, baydin2018automatic}. As gradient-based hyper-parameter optimization methods, they can be implemented as an online approach \citep{franceschi2017forward}. With the idea of auto-differentiation, learning rates can be updated in real time with the corresponding derivatives of the empirical risk \citep{almeida1998parameter}, which can be generated to all types of optimizers for deep neural networks \citep{baydin2017online}. Another step size adaptation approach called ``L4'' is based on the linearized expansion of the loss function, which focuses on minimizing the need of learning rate tunning with strong reproducible performance across multiple different architectures \citep{rolinek2018l4}. Further more, by addressing the issue of poor generalization performance of adaptive methods, dynamic bound for gradient methods was introduced to build a gradual transition between adaptive method and SGD \citep{luo2019adaptive}. 
\newline\newline
Another set of approaches train an RNN (recurrent neural network) agent to generate the optimal learning rates in the next step given the historical training information, which is known as ``learning to learn'' \citep{andrychowicz2016learning}. It empirically outperforms hand-designed optimizers in a variety of learning tasks, but another study shows that it may not be effective for long horizon \citep{lv2017learning}. The generalization ability can be improved by using meta training samples and hierachical LSTMs (Long Short-Term Memory) \citep{wichrowska2017learned}. Still there are studies focusing on incorporating domain knowledge with LSTM-based optimizers to improve the performance in terms of efficacy and efficiency \citep{fu2017neural}.
\newline\newline
The limitations of existing algorithms are mainly in the following two aspects: (a) The proposed hyper-descent only focuses on the case of global adaptation of learning rates. Even though the original paper mentions that their approach can be generalized to the case where the learning rate is an vector, it is still necessary to investigate whether different levels of parameterization could make a difference in model performance as well as training efficiency. (b) No constraints or prior knowledge for learning rates are introduced in the framework of hyper-descent, which could be essential in resolving the issue of over-parameterization when a large number of independent learning rates need to be optimized.
\newline\newline
In this study, we propose an algorithm based on existing works on hyper-descent but extend it to layer-wise, unit-wise and parameter-wise learning rates adaptation. In addition, we introduce a set of regularization techniques for learning rates for the first time to address the balance of global and local adaptation, which is also helpful in solving the issue of over-parameterization as a large number of learning rates are being learned. Although these regularizers indicate that extra hyper-parameters need to be optimized, the model performance after training could be improved with this setting in a large range of tasks. The main contribution of our study can be summarized as by the following three items:

\begin{itemize}
\item We propose an algorithm based on existing works on hyper-gradient descent but extend it to layer-wise, unit-wise and parameter-wise learning rates adaptations.
\item We introduce a set of regularization techniques for learning rates for the first time to address the balance of global and local adaptation, which is also helpful in controlling over-parameterization as a large number of learning rates are being learned.
\item We propose an algorithm for implementing the combination of adaptive learning rates in different levels for model parameter updating.
\end{itemize}
 
The structure of this chapter is organized as follows: Section~\ref{Sec:5.2} summarizes 
the related works on auto-differentiation, especially the hyper-descent (HD) algorithms. Section~\ref{Sec:5.3} explains the method implemented in extending the existing works. Section~\ref{Sec:5.4} shows the results of experiments on different learning tasks with a variety of models. Section~\ref{Sec:5.5}  discusses   the validity of the experiment results and Section~\ref{Sec:5.6} concludes the study.  

\section{Related work} \label{Sec:5.2}
This section is dedicated to reviewing the auto-differentiation and hyper-descent with detailed explanation and math formulas. In the original study of hyper-gradient descent\citep{baydin2017online}, the gradient with respect to the learning rate is calculated by using the updating rule of the model parameters in the last iteration. The gradient descent updating rule for model parameter $\theta$ can is given by Eq.~\eqref{Eq:5.1}:
\begin{equation}
   \theta_t = \theta_{t-1} - \alpha\nabla f(\theta_{t-1})
   \label{Eq:5.1}
\end{equation}
Note that $\theta_{t-1}=\theta_{t-2}-\alpha\nabla f(\theta_{t-2})$, the gradient of objective function with respect to learning rate can then by calculated:
\begin{equation}
\begin{split}
\frac{\partial f(\theta_{t-1})}{\partial \alpha} &= \nabla f(\theta_{t-1})\cdot \frac{\partial (\theta_{t-2}-\alpha\nabla f(\theta_{t-2}))}{\partial \alpha}\\
&= \nabla f(\theta_{t-1})\cdot(-\nabla f(\theta_{t-2}))
\end{split}
\end{equation}
A whole learning rate updating rule can be written as:
\begin{equation}
\begin{split}
    \alpha_t &= \alpha_{t-1}-\beta \frac{\partial f(\theta_{t-1})}{\partial \alpha}\\
    &= \alpha_{t-1}+\beta \nabla f(\theta_{t-1})\cdot \nabla f(\theta_{t-2})
    \end{split}
\end{equation}
In a more general prospective, assume that we have an updating rule for model parameters $\theta_t = u(\Theta_{t-1}, \alpha_t)$. We need to update the value of $\alpha_t$ towards the optimum value $\alpha^{*}_t$ that minimizes the expected value of the objective in the next iteration. The corresponding gradient can be written as:
\begin{equation}
    \frac{\partial \mathbb{E}[f(\theta_{t})]} {\partial \alpha_t} = \frac{\partial \mathbb{E}[f \circ u(\Theta_{t-1}, \alpha_t)]}{\partial \alpha_t} = \mathbb{E}[\nabla_{\theta} f(\theta_t)^T \nabla_{\alpha}u(\Theta_{t-1}, \alpha_t)], 
\end{equation}
where $u(\Theta_{t-1}, \alpha_t)$ denotes the updating rule of a gradient descent method. Then the additive updating rule of learning rate $\alpha_t$ can be written as:
\begin{equation}
\alpha_t = \alpha_{t-1} - \beta \tilde{\nabla}_{\theta} f(\theta_{t-1})^T \nabla_{\alpha}u(\Theta_{t-2}, \alpha_{t-1}), \label{Eq:5.2}
\end{equation}
where $\tilde{\nabla}_{\theta} f(\theta_t)$ is the noisy estimator of $\nabla_{\theta} f(\theta_t)$. On the other hand, the multiplicative rule is given by:
\begin{equation}
    \alpha_t = \alpha_{t-1}\left(1 - \beta^{'}\frac{\tilde{\nabla}_{\theta} f(\theta_{t-1})^T \nabla_{\alpha}u(\Theta_{t-2}, \alpha_{t-1})}{\left\lVert \tilde{\nabla}_{\theta} f(\theta_{t-1}) \right\rVert \left\lVert \nabla_{\alpha}u(\Theta_{t-2}, \alpha_{t-1})\right\rVert}\right). \label{Eq:5.3a}
\end{equation}
These two types of updating rules can be implemented in any optimizers including SGD and Adam, denoted by   corresponding $\theta_t = u(\Theta_{t-1}, \alpha_t)$. 

\section{Multi-level adaptation methods} \label{Sec:5.3}
In this study we propose a combination form of adaptive learning rates, where the final learning rate applied for model parameter updating is the weighted combination of different level of adaptive learning rates, while the combination weights can also be trained with back-propagation. This give the similar effect with adding regularization on learning rates with certain kind of baselines. First we introduce the learning rate adaptation in different levels.

\subsection{Layer-wise, unit-wise and parameter-wise adaptation} \label{Sec:5.3.1}
In the paper of hyper-descent\citep{baydin2017online}, the learning rate is set to be a scalar. However, to make the most of learning rate adaptation, in this study we introduce layer-wise or even parameter-wise updating rules, where the learning rate $\boldsymbol{\alpha}_t$ in each time step  is considered to be a vector (layer-wise) or even a list of matrices (parameter-wise). For the sake of simplicity, we collect all the learning rates in  a vector: $\boldsymbol{\alpha}_t = (\alpha_1, ..., \alpha_N)^T$. 
Correspondingly, the objective   $f(\boldsymbol{\theta})$ is a function of $\boldsymbol{\theta} = (\theta_1, \theta_2, ..., \theta_N)^T$, collecting all the model parameters. In this case, the derivative of the objective function $f$ with respect to each learning rate can be written as:
\begin{equation}
\begin{split}
    \frac{\partial f(\boldsymbol{\theta}_{t-1})}{\partial \alpha_{i,t-1}} &= \frac{\partial f(\theta_{1, t-1}, ..., \theta_{i, t-1}, ..., \theta_{n, t-1})}{\partial \alpha_{i,t-1}}\\
    & = \sum^N_{j=1} \frac{\partial f(\theta_{1, t-1}, ..., \theta_{i, t-1}, ...,\theta_{n, t-1})}{\partial \theta_{j, t-1}} \frac{\partial \theta_{j, t-1}}{\partial \alpha_{i, t-1}},
    \end{split}\label{Eq:5.3}
\end{equation}
where $N$ is the total number of all the model parameters. Eq.~\eqref{Eq:5.3} can be generalized to group-wise updating, where we associate a learning rate with a special group of parameters, and each parameter group is updated according to its only learning rate. Assume $\boldsymbol{\theta}_t = u(\boldsymbol{\Theta}_{t-1},\alpha)$ is the updating rule, where $\boldsymbol{\Theta}_t=\{\boldsymbol{\theta}_s\}_{s=0}^t$ and $\alpha$ is the learning rate, then the basic gradient descent method for each group $i$ gives $\boldsymbol{\theta}_{i, t} = u(\boldsymbol{\Theta}_{ t-1},\alpha_{i,t-1}) = \boldsymbol{\theta}_{i, t-1} - \alpha_{i,t-1} \nabla_{\boldsymbol{\theta}_i} f(\boldsymbol{\theta}_{t-1})$. Hence for gradient descent,
\begin{equation}
     \frac{\partial f(\boldsymbol{\theta}_{t-1})}{\partial \alpha_{i,t-1}} = \nabla_{\boldsymbol{\theta}_i} f(\boldsymbol{\theta}_{ t-1})^T\nabla_{\alpha_{i,t-1}} u(\boldsymbol{\Theta}_{t-1},\alpha_{t}) = -\nabla_{\boldsymbol{\theta}_i} f(\boldsymbol{\theta}_{ t-1})^T\nabla_{\boldsymbol{\theta}_i} f(\boldsymbol{\theta}_{t-2}). \label{Eq:5.5}
\end{equation}
Here $\alpha_{i,t-1}$ is a scalar with index $i$ at time step $t-1$, corresponding to the learning rate of the $i$th group, while the shape of $\nabla_{\boldsymbol{\theta}_i} f(\boldsymbol{\theta})$ is the same as the shape of $\boldsymbol{\theta}_i$.
\newline\newline
We particularly consider three special cases: (1) In \textbf{layer-wise adaptation}, $\boldsymbol{\theta}_i$ is the weight matrix of $i$th layer, and $\alpha_i$ is the particular learning rate for this layer. (2) In \textbf{parameter-wise adaptation}, $\theta_i$ corresponds to a certain parameter involved in the model, which can be an element of the weight matrix in a certain layer. (3) We can also introduce \textbf{unit-wise adaptation}, where $\boldsymbol{\theta}_i$ is the weight vector connected to a certain neuron, corresponding to a column or a row of the weight matrix depending on whether it is the input or the output weight vector to the neuron concerned.   \citet{baydin2017online} mentioned the case where the learning rate can be considered as a vector, which corresponds to layer-wise adaptation in this paper.

\subsection{Regularization on learning rate} \label{Sec:5.3.2}

For the model involving a large number of learning rates for different groups of parameters, the updating for each learning rate only depends on the average of a small number of examples. Therefore, when the batch size is also not large, over-parameterization is an issue to be concerned. 
\newline\newline
The idea in this study is to introduce regularization on learning rates, which can be implemented to control the flexibility of learning rate adaptation. First, for layer-wise adaptation, we can add the following regularization term to the cost function:
\begin{equation}
    L_{\text{lr\_reg\_layer}} = \lambda_{\text{layer}} \sum_l (\alpha_l - \alpha_g)^2
\end{equation}
where $l$ is the indices for each layer, $\lambda_{layer}$ is the layer-wise regularization coefficient, $\alpha_l$ and $\alpha_g$ are the layer-wise and global-wise adaptive learning rates. A large $\lambda_{\text{layer}}$ can push the learning rate of each layer towards the average learning rate across all the layers. In the extreme case, this will lead to very similar learning rates for all layers, and the algorithm will be reduced to that in \citep{baydin2017online}. 
\newline\newline
In addition, we can also consider the case where three levels of learning rate adaptations are involved, including global-wise, layer-wise and parameter-wise adaptation. If we introduce two more regularization terms to control the variation of parameter-wise learning rate with respect to layer-wise learning rate and global learning rates, the regularization loss can be written as: 
\begin{equation}
\begin{split}
L_{\text{lr\_reg\_para}} & = \lambda_{\text{layer}} \sum_l (\alpha_l - \alpha_g)^2 + \lambda_{\text{para\_layer}} \sum_l \sum_p (\alpha_{pl} - \alpha_l)^2\\
    & + \lambda_{\text{para}} \sum_l\sum_p (\alpha_{pl} - \alpha_g)^2
\end{split}
\label{Eq:5.7}
\end{equation}
where $p$ represents the index of each parameter within each layer. The second and third terms are the regularization terms pushing each parameter-wise learning rate towards the layer-wise learning rate, and the term of pushing the parameter-wise learning rate towards the global learning rates, while $\lambda_{\text{para\_layer}}$ and $\lambda_{\text{para\_layer}}$ are the corresponding regularization coefficients. 
\newline\newline
With these regularisation terms, the flexibility and variances of learning rates in different levels can be neatly controlled, while it can reduce to the basement case where a single learning rate for the whole model is used. In addition, there could still be one more regularization for improving the stability across different time steps, which can be used in the original hyper-descent algorithm where the learning rate in each time step is a scalar:
\begin{equation}
    L_{\text{lr\_reg\_ts}} = \lambda_{\text{ts}} (\alpha_{g,t} - \alpha_{g,t-1})^2
\end{equation}
where $\lambda_{ts}$ is the regularization coefficient to control the difference of learning rates between current step and the last step. With this term, the model with learning rate adaptation will be close to the model with fixed learning rate as large regularization coefficients are used. Thus, we can write the loss function of the full model as:
\begin{equation}
    L_{\text{full}} = L_{\text{model}} + L_{\text{model\_reg}} + L_{\text{lr\_reg}} + L_{\text{lr\_reg\_ts}}
\end{equation}
where $L_{\text{model}}$ and $L_{\text{model\_reg}}$ are the loss and regularization cost of basement model. $L_{\text{lr\_reg}}$ can be any among $L_{\text{lr\_reg\_layer}}$, $L_{\text{lr\_reg\_unit}}$ and $L_{\text{lr\_reg\_para}}$ depending on the specific requirement of the learning task, while the corresponding regularization coefficients can be optimized with random search for several extra dimensions.

\subsection{Updating rules for learning rates} \label{Sec:5.3.3}
Considering these regularisation terms and take layer-wise adaptation for example, the gradient of the cost function with respect to a specific learning rate $\alpha_l$ in layer $l$ can be written as:
\begin{equation}
\begin{split}
    \frac{\partial L_{\text{full}}(\theta, \alpha)}{\partial \alpha_{l, t}} &= \frac{\partial L_{\text{model}}(\theta, \alpha)}{\partial \alpha_{l, t}} +  \frac{\partial L_{\text{lr\_reg}}(\theta, \alpha)}{\partial \alpha_{l, t}}\\
    &=  \tilde{\nabla}_{\theta_l} f(\theta_{t-1})\nabla_{\alpha_{l, t-1}} u(\Theta_{t-2}, \alpha_{t-1}) + 2\lambda_{\text{layer}}  (\alpha_{l, t} - \alpha_{g, t})
    \end{split}
\end{equation}
with the corresponding updating rule by na\"ive gradient descent:
\begin{equation}
    \alpha_{l,t} = \alpha_{l,t-1} - \beta \frac{\partial L_{\text{full}}}{\partial \alpha_{l,t-1}}.
    \label{Eq:5.12}
\end{equation}
The updating rule for other types of adaptation can be derived accordingly. Notice that the time step index of layer-wise regularization term is $t$ rather than $t-1$, which ensures that we push the layer-wise learning rates towards the corresponding global learning rates of the current step. If we assume 
\begin{equation}
h_{l, t-1} = - 
\tilde{\nabla}_{\theta_l} f(\theta_{t-1})\nabla_{\alpha_{l,t-1}} u(\Theta_{t-2}, \alpha_{l, t-1})
\end{equation}
then Eq.~\eqref{Eq:5.12} can be written as:
\begin{equation}
    \alpha_{l,t} = \alpha_{l,t-1} - \beta(-h_{l, t-1} + 2 \lambda_{\text{layer}}(\alpha_{l, t} - \alpha_{g, t})). 
    \label{Eq:5.15}
\end{equation}
In Eq.~\eqref{Eq:5.15}, both sides include the term of $\alpha_{l, t}$, while the natural way to handle this is to solve for the close form of $\alpha_{t}$, which gives:
\begin{equation}
    \alpha_{l,t} = \frac{1}{1+2\beta\lambda_{\text{layer}}} [\alpha_{l,t-1} + \beta (h_{l, t-1} + 2 \lambda_{\text{layer}} \alpha_{g, t})]. 
    \label{Eq:5.16}
\end{equation}
In this formula, we still need to calculate $\alpha_{g,t}$, which is the global average learning rate in the current step. It will be even harder to calculate when there are multiple levels of learning rates, while the regularization still depends on their values in the current step. A more clean and probably computational efficient way of handling Eq.~\eqref{Eq:5.15} is to introduce approximations to get rid of $\alpha_{l, t}$ in the right hand side. If we do not consider the effect of regularization terms, the updating rule for layer-wise and global-wise learning rates can be written as:
\begin{equation}
\begin{split}
    &\hat{\alpha}_{l, t} = \alpha_{l, t-1} + \beta h_{l, t-1},\\
    &\hat{\alpha}_{g, t} = \alpha_{g, t-1} + \beta h_{g, t-1}
    \end{split}\label{Eq:5.17}
\end{equation}
where $h_{g, t-1}=-\tilde{\nabla}_{\theta} f(\theta_{t-1})\nabla_{\alpha_{g, t-1}} u(\Theta_{t-2}, \alpha_{g, t-1})$ is the global $h$ for all parameters. We define $\hat{\alpha}_{l, t}$ and $\hat{\alpha}_{g, t}$ as the ``virtual'' layer-wise and global-wise learning rates, where ``virtual'' means they are calculated based on the equation without regularization, and we do not use them directly for model parameter updating. Instead, we only use them as intermediate variables for calculating the real layer-wise learning rate for model training. 
\begin{equation}
\begin{split}
    \alpha^{*}_{l,t} &= \alpha_{l,t-1} + \beta h_{l, t-1} - 2 \beta\lambda_{\text{layer}}(\hat{\alpha}_{l, t} - \hat{\alpha}_{g,t})\\ 
    & = (1 - 2 \beta\lambda_{\text{layer}} ) \hat{\alpha}_{l, t} + 2\beta\lambda_{\text{layer}}\hat{\alpha}_{g,t}.
    \end{split} \label{Eq:5.18}
\end{equation}
Notice that in Eq.~\eqref{Eq:5.18}, the first two terms is actually a weighted average of the layer-wise learning rate $\hat{\alpha}_{l, t}$ and global learning rate $\hat{\bar{\alpha}}_{l, t}$ at the current time step. Since we hope to push the layer-wise learning rates towards the global one, the parameters should meet the constraint: $0 < 2\beta\lambda_{layer} <1$, and thus they can be optimized using hyper-parameter searching within a bounded interval. Moreover, gradient-based optimization on these hyper-parameters can also be applied. Hence both the layer-wise learning rates and the combination proportion of the local and global information can be learned with back propagation. This can be done in online or mini-batch settings. The advantage is that the learning process may be in favor of taking more account of global information in some periods, and taking more local information in some other periods to achieve the best learning performance, which is not taken into consideration by existing learning adaptation approaches. 
\newline\newline
Now consider the difference between Eq.~\eqref{Eq:5.15} and Eq.~\eqref{Eq:5.18}:
\begin{equation}
\begin{split}
    \alpha^{*}_{l,t} - \alpha_{l,t} = -2\beta\lambda_{\text{layer}}((\hat{\alpha}_{l, t} - \hat{\alpha}_{g,t}) - (\alpha_{l, t} - \alpha_{g, t})).
    \end{split}\label{Eq:5.19}
\end{equation}
Based on the setting of multi-level adaptation, on the right-hand side of Eq.~\eqref{Eq:5.19}, global learning rate is updated without regularization $\hat{\alpha}_{g, t} = \alpha_{g, t}$. For the layer-wise learning rates, the difference is given by $\hat{\alpha}_{l,t} - \alpha_{l,t} = 2\beta\lambda_{layer}(\alpha_{l,t}-\alpha_{g,t})$, which corresponds to the gradient with respect to the regularization term. Thus, Eq.~\eqref{Eq:5.19} can be rewritten as:
\begin{equation}
\begin{split}
        \alpha^{*}_{l,t} - \alpha_{l,t} &= -2\beta\lambda_{\text{layer}}(2\beta\lambda_{\text{layer}}(\alpha_{l,t}-\alpha_{g,t}))\\
        & = -4\beta^2\lambda^2_{\text{layer}}(\alpha_{l,t}-\alpha_{g,t}) = -4\beta^2\lambda^2_{\text{layer}}(1-\frac{\alpha_{g,t}}{\alpha_{l,t}})\alpha_{l,t}
        \end{split}
\end{equation}
which is the error of the virtual approximation introduced in Eq.~\eqref{Eq:5.17}. If $4\beta^2\lambda^2_{\text{layer}}<<1$ or $\frac{\alpha_{g,t}}{\alpha_{l,t}} \rightarrow 1$, this approximation becomes more accurate.
\newline\newline
Another way for handling Eq.~\eqref{Eq:5.15} is to use the learning rates for the last step in the regularization term.
\begin{equation}
    \alpha_{l,t} \approx \alpha_{l,t-1} - \beta(-h_{l, t-1} + 2 \lambda_{\text{layer}}(\alpha_{l, t-1} - \alpha_{g, t-1})). 
\end{equation}
Since we have $\alpha_{l, t} = \hat{\alpha}_{l, t}-2\beta\lambda_{\text{layer}}(\alpha_{l,t}-\alpha_{g,t})$ and $\hat{\alpha}_{l, t} = \alpha_{l, t-1} + \beta h_{l, t-1}$, using the learning rates in the last step for regularization will introduce a higher variation from term $\beta h_{l, t-1}$, with respect to the true learning rates in the current step. Thus, we consider the proposed virtual approximation works better than last-step approximation. 
\newline\newline
Similar to the two-level's case, for the three-level regularization shown in Eq.~\eqref{Eq:5.7}, we have: 
\begin{equation}
\begin{split}
    \frac{\partial L_{\text{full}}(\theta, \alpha)}{\partial \alpha_{p, t}} &= \frac{\partial L_{\text{model}}(\theta, \alpha)}{\partial \alpha_{p, t}} +  \frac{\partial L_{\text{lr\_reg}}(\alpha)}{\partial \alpha_{p, t}}\\ 
    &=  -\tilde{\nabla}_{\theta_l} f(\theta_{t-1})\nabla_{\theta_l} u(\Theta_{t-2}, \alpha_{t-1}) + 2\lambda_2 (\alpha_{p,t} - \alpha_{g,t})
    +2\lambda_3 (\alpha_{p,t}-\alpha_{l,t})
    \end{split}
\end{equation}
For the sake of simple derivation, we denote $\lambda_2 = \lambda_{\text{layer}}$, and $\lambda_3 = \lambda_{\text{para\_layer}}$ for the regularization parameters in Eq.~\eqref{Eq:5.7}. The updating rule can be written as: 
\begin{equation}
\begin{split}
    \alpha_{p,t} &= \alpha_{p, t-1}-\beta(h_p + 2\lambda_2 (\alpha_{p,t} - \alpha_{g,t})+2\lambda_3 (\alpha_{p,t}-\alpha_{l,t}))\\
    &\approx \hat{\alpha}_{p,t} (1-2\beta\lambda_2 - 2\beta\lambda_3) +2\hat{\alpha}_{l,t}\beta\lambda_3+2\hat{\alpha}_{g,t}\beta\lambda_2
    \end{split}
\end{equation}
where we assume that $\hat{\alpha}_{p,t}$, $\hat{\alpha}_{l,t}$, $\hat{\alpha}_{g,t}$ are independent variables. Define 
\[\gamma_1 = 1-2\beta\lambda_2 - 2\beta\lambda_3,\, \gamma_2 = 2\beta\lambda_3,\,  \gamma_3 = 2\beta \lambda_2,\]
we still have:
\begin{equation}
\begin{split}
    &\alpha = \gamma_1\alpha_p + \gamma_2\alpha_l + \gamma_3\alpha_g,\\
    &\gamma_1 + \gamma_2 + \gamma_3 = 1.
    \end{split}
\end{equation}
Therefore, in the case of three level learning rates adaptation, the regularization effect can still be considered as applying the weighted combination of different levels of learning rates. This conclusion is invariant of the signs in the absolute operators in Eq.~\eqref{Eq:5.17}. \newline\newline
In general, we can organize all the learning rates in a tree structure. For example, in three level case above, $\alpha_g$ will be the root node, while $\{\alpha_l\}$ are the children node at level 1 of the tree and $\{\alpha_{lp}\}$ are the children node of $\alpha_l$ as leave nodes at level three of the tree. In a general case, we assume there are $L$ levels in the tree. Denote the set of all the paths from the root node to each of leave nodes as $\mathcal{P}$ and a path is denoted by $p = \{\alpha_1, \alpha_2, ..., \alpha_L\}$ where $\alpha_1$ is the root node and $\alpha_L$ is the leave node on the path. On this path, denote $\text{ancestors}(i)$ all the acenstor nodes of $\alpha_i$ along the path, i.e., $\text{ancestors}(i) = \{\alpha_1, ..., \alpha_{i-1}\}$. We will construct a regularizer to push $\alpha_i$ towards each of its parents. Then the regularization can be written as
\begin{align}
L_{\text{lr\_reg}} = \sum_{p\in\mathcal{P}}\sum_{\alpha_i\in p}\sum_{\alpha_j\in\text{acenstors}(i)}\lambda_{ij}(\alpha_i - \alpha_j)^2. \label{eq:5.overall}
\end{align}
Under this pair-wise $L_2$ regularization, the updating rule for any leave node learning rate $\alpha_L$ can be given by the following theorem
\begin{thm} Under virtual approximation, effect of adding pair-wise $L_2$ regularization on different levels of adaptive learning rates $L_{\text{reg}} = \sum^n_i \sum^n_{j<i} \lambda_{ij}\|\alpha_i - \alpha_j\|^2_2$ is equal to performing a weighted linear combination of virtual learning rates in different levels $\alpha^{*} = \sum^n_i\gamma_i \alpha_{i}$ with $\sum^n_i \gamma_i = 1$, where each component $\alpha_i$ is calculated by assuming there is no regularization.
\label{thm:1}
\end{thm}

\textit{Remarks:} Theorem~\ref{thm:1} actually suggests that the similar updating rule can be obtained for the learning rate at the any level on the path. All these have been demonstrated in Algorithm~\ref{Alg:3.1} for the three level case.
\begin{proof} Consider the learning regularizer
\begin{align}
L_{\text{lr\_reg}}(\alpha) = \sum_{p\in\mathcal{P}}\sum_{\alpha_i\in p}\sum_{\alpha_j\in\text{parents}(i)}\lambda_{ij}(\alpha_i - \alpha_j)^2. \label{eq:5.overall_appendix}
\end{align}
To apply hyper-gradient descent method to update the learning rate $\alpha_L$ at level $L$, we need to work the derivative of $L_{\text{lr\_reg}}$ with respect to $\alpha_L$, the terms in \eqref{eq:5.overall_appendix} involving $\alpha_L$ are only $(\alpha_i - \alpha_j)^2$ where $\alpha_j$ is an ancestor on the path from the root to the leave node $\alpha_L$. Hence
\begin{equation}
\begin{split}
    \frac{\partial L_{\text{full}}(\boldsymbol{\theta}, \alpha)}{\partial \alpha_{L, t}} &= \frac{\partial L_{\text{model}}(\boldsymbol{\theta}, \alpha)}{\partial \alpha_{L, t}} +  \frac{\partial L_{\text{lr\_reg}}(\alpha)}{\partial \alpha_{L, t}}\\ 
    &=  -\tilde{\nabla}_{\boldsymbol{\theta}_L} f(\boldsymbol{\theta}_{t-1})^T\nabla_{\boldsymbol{\theta}_L} u(\Theta_{t-2}, \alpha_{t-1}) + \sum_{\alpha_j\in \text{acenstors}(L)} 2\lambda_{Lj} (\alpha_{L,t} - \alpha_{j,t}).
    \end{split}
\end{equation}
As there are exactly $L-1$ ancestors on the path, we can simply use the index $j = 1, 2, ..., L-1$. The corresponding updating function for $\alpha_{n,t}$ is:
\begin{equation}
\begin{split}
    \alpha_{L,t} &= \alpha_{n, t-1}-\beta(h_L + \sum_{j=1}^{L-1} 2\lambda_{Lj} (\alpha_{L,t} - \alpha_{j,t}))\\
    &\approx \hat{\alpha}_{L,t} (1- 2\beta\sum_{j=1}^{L-1}\lambda_{Lj}\alpha_{n,t}) + \sum_{j=1}^{L-1} (2\beta\lambda_{Lj}\hat{\alpha}_{j,t}))\\
    &=\sum_{j=1}^L \gamma_{j} \hat{\alpha}_{j,t}.
    \end{split}
\end{equation}
where
\begin{equation}
\begin{split}
&\gamma_L = 1- 2\beta\sum_{j=1}^{L-1}\lambda_{Lj},\\
&\gamma_j = 2\beta\lambda_{Lj}, \quad \text{for } j =1, 2, ..., L-1.
\end{split}
\end{equation}
This form satisfies $\alpha^{*}_L = \sum^L_{j=1}\gamma_j \hat{\alpha}_{j}$ with $\sum^L_{j=1} \gamma_j = 1$. This completes the proof.
\end{proof}

\subsection{Prospective of learning rate combination}\label{Sec:5.3.4}
Motivated by the analytical derivation in Section ~\ref{Sec:5.3.3}, we can consider the combination of adaptive learning rates in different levels as a substitute of regularization on the differences of learning rates. As a simple case, the combination of global-wise and layer-wise adaptive learning rates can be written as:
\begin{equation}
    \alpha_t = \gamma_1 \hat{\alpha}_{l, t} + \gamma_2 \hat{\alpha}_{g,t},
\end{equation}
where $\gamma_1 + \gamma_2 = 1$ and $\gamma_1\geq 0$, $\gamma_2\geq 0$. In a general form, assume that we have $n$ levels, which could include global-level, layer-level, unit-level and parameter-level, etc, we have:
\begin{equation}
    \alpha_t = \sum_{i=1}^n\gamma_i \hat{\alpha}_{i, t}.
    \label{Eq:5.25}
\end{equation}
In a more general form, we can implement non-linear models such as neural networks to model the final adaptive learning rates with respect of the learning rates in different levels.
\begin{equation}
    \alpha_t = g(\hat{\alpha}_{1, t}, \hat{\alpha}_{2, t} ... \hat{\alpha}_{n, t}; \theta),
\end{equation}
where $\theta$ is the vector of parameters of the non-linear model. In this study, we treat the combination weights $\{\gamma_1, ..., \gamma_n\}$ as trainable parameters as demonstrated in Eq.~\eqref{Eq:5.25}. Figure~\ref{Fig:4.CAM-HD} gives an illustration of the linear combination of three-level hierarchical learning rates.
\begin{figure}[th] 
\begin{center}
 \includegraphics[width=1.0\linewidth]{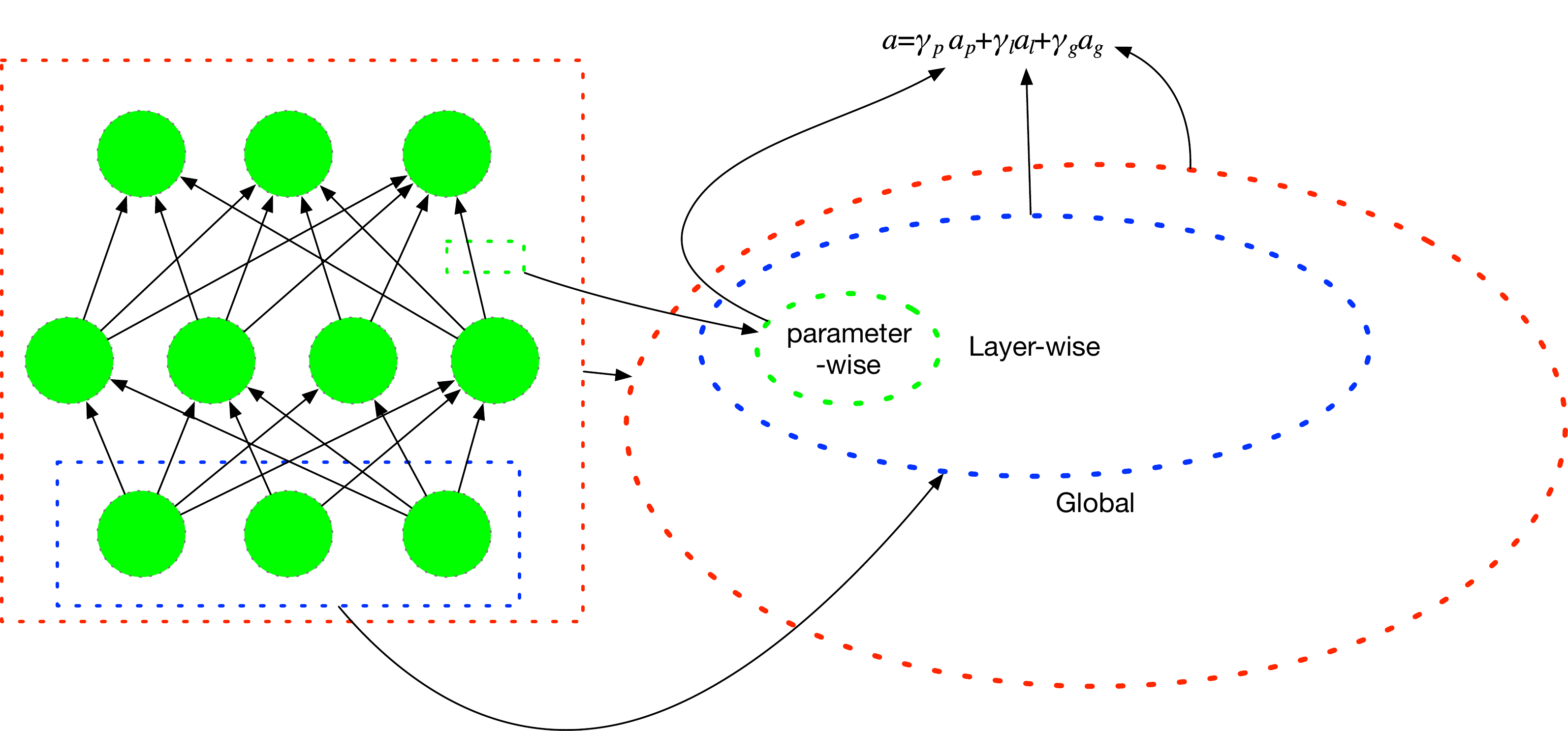}
 \caption{The diagram of a three-level learning rate combination} \label{Fig:4.CAM-HD}
\end{center} 
\end{figure}
In fact, we only need these different levels of learning rate have a hierarchical relationship, which means the selection of component levels is not fixed. For example, in feed-forward neural networks, we can use parameter level, unit-level, layer level and global level. For recurrent neural networks, the corresponding layer level can either be the ``layer of gate'' within the cell structure such as LSTM and GRU, or the whole cell in a particular RNN layer. Especially, by ``layer of gate'' we mean the parameters in each gate of a cell structure share a same learning rate. Meanwhile, for convolutional neural network, we can further introduce ``filter level'' to replace layer-level if their is no clear layer structure, where the parameters in each filter will share a same learning rate.
\newline\newline
As the real learning rates implemented in model parameter updating is a weighted combination, the corresponding Hessian matrices cannot be directly used for learning rate updating. If we take the gradients of the loss with respect to the combined learning rates, and use this to update the learning rate for each parameter, the procedure will be reduced to parameter-wise learning rate updating. To address this issue, we first break down the gradient by the combined learning rate to three levels, use each of them to updated the learning rate in each level, and then calculate the combination by the updated learning rates. Especially, $h_{p, t}$, $h_{l, t}$ and $h(g,t)$ are calculated by the gradients of model losses without regularization, as is shown in Eq. ~\eqref{Eq:5.30}. 
\begin{equation}
\begin{split}
  h_{p, t} &= \frac{\partial f(\boldsymbol{\theta}, \alpha)}{\partial \alpha_{p, t}} = 
  -\nabla_{\boldsymbol{\theta}}f(\boldsymbol{\theta}_{t-1}, \alpha)|_p \cdot\nabla_{\alpha}u(\boldsymbol{\Theta}_{t-2}, \alpha)|_p\\
  h_{l, t} &= \frac{\partial f(\boldsymbol{\theta}, \alpha)}{\partial \alpha_{l, t}} = 
  -\text{tr}(\nabla_{\boldsymbol{\theta}}f(\boldsymbol{\theta}_{t-1}, \alpha)|_l^T \nabla_{\alpha}u(\boldsymbol{\Theta}_{t-2}, \alpha)|_l)\\
  h_{g, t} &= \frac{\partial f(\boldsymbol{\theta}, \alpha)}{\partial \alpha_{t}} = 
  -\sum_{l=1}^n\text{tr}(\nabla_{\boldsymbol{\theta}}f(\boldsymbol{\theta}_{t-1}, \alpha)|_l^T \nabla_{\alpha}u(\boldsymbol{\Theta}_{t-2}, \alpha)_l)
\end{split}\label{Eq:5.30}
\end{equation}
where $h_t = \sum_l h_{l,t} = \sum_p h_{p, t}$ and $h_{l,t} = \sum_{p\in l\text{th layer}} h_{p}$ and $f(\theta, \alpha)$ corresponds to the model loss $L_{model}(\theta, \alpha)$ in Section ~\ref{Sec:5.3.2}. 
Algorithm ~\ref{Alg:3.1} is the full updating rules for the newly proposed optimizer with three levels, which can be denoted as combined adaptive multi-level hyper-gradient descent (CAM-HD).
\begin{algorithm}
\SetAlgoLined
 \bf{input}: $\alpha_{0}$, $\beta$, $\delta$, $T$\newline
 \bf{initialization}: $\theta_0$, $\gamma_{1,0}$, $\gamma_{2,0}$, $\gamma_{3,0}$, $\alpha_{p,0}$, $\alpha_{l,0}$, $\alpha_{0}, \alpha^{*}_{l,0} = \gamma_{1,0} \alpha_{p,0} + \gamma_{2,0} \alpha_{l,0} + \gamma_{3,0} \alpha_0$
 \newline
  \For{$t \in {1, 2, ..., T}$} {
  $g_t = \nabla_{\theta}f(\theta, \alpha)$\newline
  $h_{p, t} = \frac{\partial f(\theta, \alpha)}{\partial \alpha_{p, t}} = 
  -\nabla_{\theta}f(\theta_{t-1}, \alpha)|_p \cdot\nabla_{\alpha}u(\Theta_{t-2}, \alpha)|_p
  $ \newline
  $h_{l, t} = \frac{\partial f(\theta, \alpha)}{\partial \alpha_{l, t}} = 
  -\text{tr}(\nabla_{\theta}f(\theta_{t-1}, \alpha)|^T_l \nabla_{\alpha}u(\Theta_{t-2}, \alpha)|_l)
  $ \newline
  $h_{g,t} = \frac{\partial f(\theta, \alpha)}{\partial \alpha_{t}} = 
  -\sum_{l=1}^n\text{tr}(\nabla_{\theta}f(\theta_{t-1}, \alpha)|_l^T \nabla_{\alpha}u(\Theta_{t-2}, \alpha))|_l)
  $ \newline
  $\alpha_{p,t} = \alpha_{p,t-1} - \beta_p \frac{\partial f(\theta_{t-1})}{\partial \alpha^{*}_{p, t-1}}\frac{\partial \alpha^{*}_{p, t-1}}{\partial \alpha_{p, t-1}} = \alpha_{p,t-1} - \beta_p \gamma_{1, t-1}h_{p,t}$\newline
  $\alpha_{l,t} = \alpha_{l,t-1} - \beta_l\sum_p \frac{\partial f(\theta_{t-1})}{\partial \alpha^{*}_{p, t-1}}\frac{\partial \alpha^{*}_{p, t-1}}{\partial \alpha_{l, t-1}}= \alpha_{l,t-1} - \beta_l \gamma_{2, t-1} \sum_p h_{p,t} = \alpha_{l,t-1} - \beta_l \gamma_{2, t-1} h_{l,t}$\newline
  $\alpha_{t}=\alpha_{t-1}-\beta_g \sum_l \sum_p \frac{\partial f(\theta)}{\partial \alpha^{*}_{p,t-1}}\frac{\partial \alpha^{*}_{p, t-1}}{\partial \alpha_{t-1}}= \alpha_{t-1}-\beta_g \gamma_{3, t-1} h_{g, t}$ \newline
  $\alpha^{*}_{p,t} = \gamma_{1,t-1} \alpha_{p,t} + \gamma_{2,t-1} \alpha_{l,t} + \gamma_{3,t-1} \alpha_t $\newline
  $\gamma_{1, t} = \gamma_{1, t-1}-\delta \frac{\partial L}{\partial \gamma_{1, t-1}} = \gamma_{1, t-1}-\delta\sum_p \frac{\partial L}{\partial \alpha^{*}_{p, t-1}}\frac{\partial \alpha^{*}_{p, t-1}}{\partial \gamma_{1, t-1}} = \gamma_{1, t-1}-\delta\alpha_{p,t-1}\sum_p \frac{\partial L}{\partial \alpha^{*}_{p, t-1}}$ \newline
  $\gamma_{2, t} = \gamma_{2, t-1}-\delta \frac{\partial L}{\partial \gamma_{2, t-1}} = \gamma_{2, t-1}-\delta\sum_p \frac{\partial L}{\partial \alpha^{*}_{p, t-1}}\frac{\partial \alpha^{*}_{p, t-1}}{\partial \gamma_{2, t-1}} = \gamma_{1, t-1}-\delta\alpha_{l,t-1}\sum_p \frac{\partial L}{\partial \alpha^{*}_{p, t-1}} $ \newline
  $\gamma_{3, t} = \gamma_{3, t-1}-\delta \frac{\partial L}{\partial \gamma_{3, t-1}} = \gamma_{3, t-1}-\delta\sum_p\frac{\partial L}{\partial \alpha^{*}_{p, t-1}}\frac{\partial  \alpha^{*}_{p, t-1}}{\partial \gamma_{3, t-1}}  =  \gamma_{3, t-1}-\delta \alpha_{t-1}\sum_p \frac{\partial L}{\partial \alpha^{*}_{p, t-1}} $  \newline
  $\gamma_1 = \gamma_1/(\gamma_1 + \gamma_2 + \gamma_3)$,
  $\gamma_2 = \gamma_1/(\gamma_1 + \gamma_2 + \gamma_3)$,
  $\gamma_3 = \gamma_1/(\gamma_1 + \gamma_2 + \gamma_3)$\newline
  $m_t = \phi_t(g_1,...g_t)$\newline
  $V_t = \psi_t(g_1,...g_t)$\newline
  $\theta_t = \theta_{t-1} - \alpha^{*}_{p,t}m_t/\sqrt{V_t}$ \newline
 }
\Return $\theta_T$, $\gamma_{1,T}$, $\gamma_{2,T}$, $\gamma_{3,T}$, $\alpha_{p,T}$, $\alpha_{l,T}$, $\alpha_{T}$
 \caption{Updating rule of three-level CAM-HD}\label{Alg:3.1}
\end{algorithm}
where we introduce the general form of gradient descent based optimizers\citep{reddi2019convergence, luo2019adaptive}. For SGD, $\phi_t(g_1,...g_t)=g_t$ and $\psi_t(g_1,...g_t)=1$, while for Adam, $\phi_t(g_1,...g_t)= (1-\beta_1)\Sigma_{i=1}^t \beta^{t-1}_1 g_i$ and $\psi_t(g_1,...g_t)=(1-\beta_2)\text{diag}(\Sigma_{i=1}^t \beta_2^{t-1} g_i^2)$. Notice that in each updating time step of Algorithm ~\ref{Alg:3.1}, we re-normalize the combination weights $\gamma_1$, $\gamma_2$ and $\gamma_3$ to make sure that their summation is always 1 even after updating with stochastic gradient-based methods. An alternative way of doing this is to implemented softmax, which require an extra set of intermediate variables $c_p$, $c_l$ and $c_g$ following: $\gamma_p=\text{softmax}(c_p)= \exp^{c_p}/(\exp^{c_p}+\exp^{c_l}+\exp^{c_g})$, etc. Then the updating of $\gamma$s will be convert to the updating of $c$s during training. In addition, the training of $\gamma$s can also be extended to multi-level cases, which means we can have different combination weights in different layers. For the updating rates $\beta_p$, $\beta_l$ and $\beta_g$ of the learning rates in different level, we set:
\begin{equation}
    \beta_p = n_p\beta=\beta,\, \beta_l = n_l\beta,\, \beta_g = n\beta
\end{equation}
where $\beta$ is a shared parameter. This setting will make the updating steps of learning rates in different levels be in the same scale considering the difference in the number of parameters involved in $h_{p,t}$, $h_{l,t}$, $h_{g,t}$. If we take average based on the number of parameters in Eq.~\eqref{Eq:5.30} at first, this adjustment is not required.

CAM-HD is a higher-level adaptation approach, which can be applied with any gradient-based updating rules and advanced adaptive optimizers. For exmaple, it can be merged with Adabound by adding a parameter-wise clipping procedure \citep{luo2019adaptive}:
\begin{equation}
    \hat{\eta} = \text{Clip}(\alpha^{*}/\sqrt{V_t}, \eta_l, \eta_u),\, \eta_t = \hat{\eta}/\sqrt{t}
    \label{Eq:clip}
\end{equation}
where $\alpha^{*}$ is the final step-size by original CAM-HD, $\eta_l$ and $\eta_u$ are the lower and upper bounds in adabound. $\eta_t$ can be applied in replacing $\alpha^{*}_{p,t}/\sqrt{V_t}$ in our algorithm for merging two methods to so called ``Adabound-CAM-HD''. In the experiment part, we will follow the original paper to set $\eta_l(t)=0.1-\frac{0.1}{(1-\beta_2)t+1}$ and $\eta_u(t)=0.1+\frac{0.1}{(1-\beta_2)t+1}$ for both Adabound and Adabound-CAM-HD.

\subsection{Convergence analysis:}\label{Sec:5.3.5}
The proposed CMA-HD is not an independent optimization method, which can be applied in any kinds of gradient-based methods. Its convergence properties highly depends on the base optimizer that is applied. Here we provide an analysis based on the general prospective of learning rate adaptation \citep{baydin2017online,karimi2016linear}. We have learned that for global-wise learning rate adaptation, if we assume that $f$ is convex and L-Lipschitz smooth with $\|\nabla f(\theta)\| < M$ for some fixed $M$ and all $\theta$, the learning rate $\alpha_t$ satisfies:
\begin{equation}
\begin{split}
    |\alpha_t|&\leq |\alpha_0|+\beta\sum_{i=0}^{t-1}|\nabla f(\theta_{i+1})^T \nabla f(\theta_i)|\leq |\alpha_0|+\beta\sum_{i=0}^{t-1}\|\nabla f(\theta_{i+1})\|\|\nabla f(\theta_i)\|\\
    &\leq |\alpha_0|+t\beta M^2
    \end{split}
\end{equation}
where $\alpha_0$ is the initial value of $\alpha$, and $\beta$ is the updating rate for hyper-gradient descent. By introducing $\kappa_{p,t} = \tau(t)\alpha^{*}_{p,t} + (1-\tau(t))\alpha_0$, where the function $\tau(t)$ is selected to satisfy $\tau(t)\rightarrow 0$ as $t\rightarrow \infty$, we have the following convergence theorem. 
\newline
\begin{thm}Convergence under certain assumptions about $f$
Suppose that $f$ is convex and L-Lipschitz smooth with $\|\nabla f(\theta)\| < M$ for some fixed $M$ and all $\theta$. Then $\theta_t\rightarrow \theta^{*}$ if $\alpha^{\infty}<1/L$ and $t\cdot\tau(t)\rightarrow 0 $ as $t \rightarrow \infty$, where the $\theta_t$ are generated accroding to (non-stochastic) gradient descent.
\end{thm}
\noindent The proposed CMA-HD is not an independent optimization method, which can be applied in any kinds of gradient-based updating rules. Its convergence properties highly depends on the base optimizer that is applied. By referring the discussion on convergence in \citep{baydin2017online}, if we introduce $\kappa_{p,t} = \tau(t)\alpha^{*}_{p,t} + (1-\tau(t))\alpha_{\infty}$, where the function $\tau(t)$ is selected to satisfy $t\tau(t)\rightarrow 0$ as $t\rightarrow \infty$, and $\alpha_{\infty}$ is a selected constant value. Then we demonstrate the convergence analysis for the three level case in the following theorem, where $\nabla_p$ is the the gradient of target function w.r.t. a model parameter with index $p$, $\nabla_l$ is the average gradient of target function w.r.t. a parameters in a layer with index $l$, and $\nabla_g$ is the global average gradient of target function w.r.t. all model parameters.
\newline
\begin{thm}[Convergence under mild assumptions about $f$]
Suppose that $f$ is convex and L-Lipschitz smooth with $\|\nabla_p f(\theta)\| < M_p$, $\|\nabla_l f(\theta)\| < M_l$, $\|\nabla_g f(\theta)\| < M_g$ for some fixed $M_p$, $M_l$, $M_g$ and all $\theta$. Then $\theta_t\rightarrow \theta^{*}$ if $\alpha_{\infty}<1/L$ where $L$ is the Lipschitz constant for all the gradients and $t\cdot\tau(t)\rightarrow 0 $ as $t \rightarrow \infty$, where the $\theta_t$ are generated according to (non-stochastic) gradient descent.
\label{thm:2}
\end{thm}
\begin{proof}
We take three-level's case discussed in Section~\ref{Sec:5.3} for example, which includes global level, layer-level and parameter-level. Suppose that the target function $f$ is convex, L-Lipschitz smooth in all levels, which gives for all $\theta_1$ and $\theta_2$:
\begin{equation}
\begin{split}
    &||\nabla_p f(\theta_1) - \nabla_p f(\theta_2) ||\leq L_p|| \theta_1 -\theta_2||\\
    &||\nabla_l f(\theta_1) - \nabla_l f(\theta_2) ||\leq L_l|| \theta_1 -\theta_2||\\
    &||\nabla_g f(\theta_1) - \nabla_g f(\theta_2) ||\leq L_g|| \theta_1 -\theta_2||\\
    &L = \max\{L_p, L_l, L_g\}
    \end{split}
\end{equation}
and its gradient with respect to parameter-wise, layer-wise, global-wise parameter groups satisfy $\|\nabla_p f(\theta)\| < M_p$, $\|\nabla_l f(\theta)\| < M_l$, $\|\nabla_g f(\theta)\| < M_g$ for some fixed $M_p$, $M_l$, $M_g$ and all $\theta$. Then the effective combined learning rate for each parameter satisfies: 
\begin{equation}
\begin{split}
    |\alpha^{*}_{p,t}| &= |\gamma_{p,t-1} \alpha_{p,t} + \gamma_{l,t-1} \alpha_{l,t} + \gamma_{g,t-1} \alpha_t|\\
    &\leq (\gamma_{p,t-1}+ \gamma_{l,t-1} + \gamma_{g,t-1})\alpha_0 + \beta\sum_{i=0}^{t-1}\left(\gamma_{p,t-1}n_p\max_p\{|\nabla f(\theta_{p,i+1})^T \nabla f(\theta_{p,i})|\}\right.\\
    &\left.+ \gamma_{l,t-1}n_l \max_l\{|\nabla f(\theta_{l,i+1})^T \nabla f(\theta_{l,i})|\} + \gamma_{g,t-1}|\nabla f(\theta_{g, i+1})^T \nabla f(\theta_{g, i})|\right)\\
    & \leq \alpha_0 + \beta\sum_{i=0}^{t-1}\left(\gamma_{p,t-1}n_p \max_p\{\|\nabla f(\theta_{p,i+1})\| \|\nabla f(\theta_{p,i})\|\}\right.\\
    &\left.+ \gamma_{l,t-1}n_l\max_l\{\|\nabla f(\theta_{l,i+1})\| \|\nabla f(\theta_{l,i})\|\} + \gamma_{g,t-1}\|\nabla f(\theta_{g, i+1})\| \| \nabla f(\theta_{g, i})\|\right)\\
    &\leq \alpha_0 +  t\beta (n_p M_p^2 + n_l M_l^2 + M_g^2)
     \end{split}
\end{equation}
where $\theta_{p,i}$ refers to the value of parameter indexed by $p$ at time step $i$, $\theta_{l,i}$ refers to the set/vector of parameters in layer with index $l$ at time step $i$, and $\theta_{g,i}$ refers to the whole set of model parameters at time step $i$. In addition, $n_p$ and $n_l$ are the total number of parameters and number of the layers, and we have applied $0<\gamma_p, \gamma_l, \gamma_g<1$. This gives an upper bound for the learning rate in each particular time step, which is $O(t)$ as $t\rightarrow \infty$. By introducing $\kappa_{p,t} = \tau(t)\alpha^{*}_{p,t} + (1-\tau(t))\alpha_{\infty}$, where the function $\tau(t)$ is selected to satisfy $t\tau(t)\rightarrow 0$ as $t\rightarrow \infty$, so we have $\kappa_{p,t}\rightarrow \alpha_{\infty}$ as $t\rightarrow \infty$. If $\alpha_{\infty}<\frac{1}{L}$, 
 for larger enough $t$, we have 
$1/(L+1)< \kappa_{p,t} <1/L$, and the algorithm converges when the corresponding gradient-based optimizer converges for such a learning rate under our assumptions about $f$. This follows the discussion in \citep{karimi2016linear, sun2019optimization}.
\end{proof}
\noindent When we introduce $\kappa_{p,t}$ instead of $\alpha^{*}_{p,t}$ in Algorithm ~\ref{Alg:3.1}, the corresponding gradients $\frac{\partial L(\theta)}{\partial \alpha^{*}_{p, t-1}}$ will also be replaced by $\frac{\partial L(\theta)}{\partial \kappa^{*}_{p,t-1}}\frac{\partial \kappa^{*}_{p,t-1}}{\partial \alpha^{*}_{p, t-1}} = \frac{\partial L(\theta)}{\partial \kappa^{*}_{p,t-1}}\tau(t)$. 

\begin{thm}[Convergence of Adabound-CAM-HD]
Let $\{\theta_t\}$ and $\{V_t\}$ be the sequences obtained from the modified Algorithm~\ref{Alg:3.1} for Adabound-CAM-HD discussed in Section~\ref{Sec:5.3.4}. The optimizer parameters in Adam satisfy $\beta_1=\beta_{11}$, $\beta_{1t}\leq \beta_1$ for all $t\in [T]$ and $\beta_1<\sqrt{\beta_2}$. Suppose $f$ is a convex target function on $\Theta$, $\eta_l(t)$ and $\eta_u (t)$ are the lower and upper bound function, $\eta_l(t+1)\geq \eta_l(t)>0$, $\eta_u(t+1)<\eta_u (t)$. As $t\rightarrow\infty$, $\eta_l(t)\rightarrow \alpha^{*}$, $\eta_u(t)\rightarrow \alpha^{*}$. $L_{\infty}=\eta_l(1)$ and $R_{\infty}=\eta_u(1)$. Assume that $||\theta_1-\theta_2||_{\infty}\leq D_{\infty}$ for all $\theta_1, \theta_2 \in \Theta$ and $||\nabla f_t(\theta)||\leq G_2$ for all $t\in [T]$ and $\theta \in \Theta$. For $\theta_t$ generated using Adabound-CAM-HD algorithm, the regret function $R(T) = \sum_{t=1}^T f_t(\theta_t)-\min_{\theta\in \Theta}\sum_{t=1}^Tf_t(\theta)$ is upper bounded by $O(\sqrt{T})$. 
\end{thm}
Due to the clipping procedure in Eq.~\eqref{Eq:clip}, the $\eta_t$ for parameter updating satisfies $L_{\infty}\leq \sqrt{t}||\eta_t||_{\infty}\leq R_{\infty}$. Hence, the proof of convergence of Adabound in \citep{luo2019adaptive} is also valid for Adabound-CAM-HD, ensuring that it achieves a high level of adaptiveness with a good convergence property. Notice that in \citep{savarese2019convergence}, it is recommended to suppose $\frac{t}{\eta_l(t)}-\frac{t-1}{\eta_u(t-1)}\leq M$ for all $t\in [T]$ as a correction. As the effective parameter-wise updating rates and corresponding gradients may change after clipping, the updating rules for other variable should be adjusted accordingly. 

\section{Experiments} \label{Sec:5.4}

We use the feed-forward neural network models and different types of convolutions neural networks on multiple benchmark datasets to compare with existing baseline optimizers.
For each learning task, the following optimizers will be applied: (a) standard baseline optimizers such as Adam and SGD; (b) hyper-gradient descent in \citep{baydin2017online}; (c) L4 stepsize adaptation for standard optimizers \citep{rolinek2018l4}; (d) Adabound optimizer \citep{luo2019adaptive}; (e) RAdam optimizer \citep{liu2019variance}; and (f) the proposed adaptive combination of different levels of hyper-descent. The implementation of (b) is based on the code provided with the original paper. One NVIDIA Tesla V100 GPU with 16G Memory 61 GB RAM and two Intel Xeon 8 Core CPUs with 32 GB RAM are applied. The program is built in Python 3.5.1 and Pytorch 1.0 \citep{subramanian2018deep}. For each experiment, we provide both the average curves and standard error bars for ten runs.

\subsection{Hyper-parameter Tuning}\label{Sec:5.4.1}

To compare the effect of CAM-HD with baseline optimizers, we first do hyperparameter tuning for each learning task by referring to related papers \citep{kingma2014adam, baydin2017online, rolinek2018l4, luo2019adaptive} as well as implementing an independent grid search \citep{bergstra2011algorithms, feurer2019hyperparameter}. We mainly consider hyper-parameters including batch size, learning rate, and other optimizer parameters for models with different architectures. Other settings in our experiments follow open-source benchmark models. The search space for batch size is the set of $\{2^n\}_{n=3,...,9}$, while the search space for learning rate, hyper-gradient updating rate and combination weight updating rate (CAM-HD-lr) are $\{10^{-1},10^{-2},...,10^{-4}\}$, $\{10^{-1},10^{-2},...,10^{-10}\}$ and $\{0.1, 0.03, 0.01, 0.003, 0.001, 0.0003, 0.0001\}$, respectively. 
The selection criterion is the 5-fold cross-validation loss by early-stopping at the patience of 3 \citep{prechelt1998early}. The optimized hyper-parameters for the tasks in this paper are given in Table~\ref{tab:2}. For training ResNets with SGDN, we will apply a step-wise learning rate decay schedule as in \citep{luo2019adaptive, liu2019variance}. Notice that although the hyper-parameters are tuned, it does not mean that the model performance is sensitive to each hyper-parameter.
\begin{table*}[htbp]
  \centering{\scriptsize
  \caption{Hyperparameter Settings for Experiments}
    \begin{tabular}{cccccccc}
    \toprule
    \multicolumn{1}{p{5.5em}}{Architecture} & \multicolumn{1}{p{4.585em}}{Dataset} & \multicolumn{1}{p{3em}}{Batch size} & \multicolumn{1}{p{3.5em}}{lr (SGD/SGDN)} & \multicolumn{1}{p{5em}}{lr (Adam)} & \multicolumn{1}{p{6em}}{Hyper-grad lr (SGD/SGDN)} & \multicolumn{1}{p{6em}}{Hyper-grad lr (Adam)} & \multicolumn{1}{p{5em}}{CAM-HD-lr} \\
    \midrule
    MLP 1 & \multirow{3}[2]{*}{MNIST} & 32    &   -    & 0.0003 &   -   & 1.00E-07 & 0.01 \\
    MLP 2 &       & 64    &   -    & 0.001 &    -   & 1.00E-07 & 0.01 \\
    MLP 3 &       & 128   &   -    & 0.001 &    -   & 1.00E-07 & 0.01 \\
    \midrule
    \multirow{3}[2]{*}{LeNet-5} & MNIST & 256   &   -    & 0.001 & 1.00E-03 & 1.00E-08 & 0.03 \\
          & CIFAR10 & 256   &   -   & 0.001 & 1.00E-03 & 1.00E-08 & 0.03 \\
          & SVHN  & 128   &   -   & 0.001 & 1.00E-03 & 1.00E-08 & 0.03 \\
    \midrule
    ResNet-18 & \multirow{2}[2]{*}{CIFAR10} & 256   & 0.1   & 0.001 & 1.00E-06 & 1.00E-08 & 0.001 \\
    ResNet-34 &       & 256   & 0.1   & 0.001 & 1.00E-06 & 1.00E-08 & 0.001 \\
    \bottomrule
    \end{tabular}
  \label{tab:2}}
\end{table*}
For training ResNets with SGDN, we will apply a step-wise learning rate decay schedule as in \citep{luo2019adaptive, liu2019variance}. Notice that although the hyper-parameters are tuned, it does not mean that the model performance is sensitive to each of them.

\subsection{Combination Ratio and Model Performances} \label{Sec:5.4.2}
First, we perform a study on the initialization of the combination weights different level learning rates in the framework of CAM-HD. The simulations are based on image classification tasks on MNIST and CIFAR10 \citep{lecun1998gradient, krizhevsky2009learning}. We use full training sets of MNIST and CIFAR10 for training and full test sets for validation. One feed-forward neural network with three hidden layers of size [100, 100, 100] and two convolutional network models, including LeNet-5 \citep{lecun2015lenet} and ResNet-18 \citep{he2016deep}, are implemented. In each case, two levels of learning rates are considered, which are the global and layer-wise adaptation for FFNN, and global and filter-wise adaptation for CNNs. For LeNet-5 and FFNN, Adam-CAM-HD with fixed and trainable combination weights is implemented, while for ResNet-18, both Adam-CAM-HD and SGDN-CAM-HD with fixed and trainable combination weights are implemented in two independent simulations. We change the initialized combination weights of two levels in each case to see the change of model performance in terms of test classification accuracy at epoch 30 for FFNN, and at epoch 10 for LeNet-5 and ResNet-18. Also we compare CAM-HD methods with baseline Adam and SGDN methods in terms of test accuracy after the same epochs of training. Other hyper-parameters are optimized based on Section~\ref{Sec:5.4.1}. We conduct 10 runs at each combination ratio and draw the average accuracies and corresponding error bars (standard errors). The result is given in Figure~\ref{Fig:1}, 
\begin{figure*}[th] 
\begin{center}
 \includegraphics[width=1.0\linewidth]{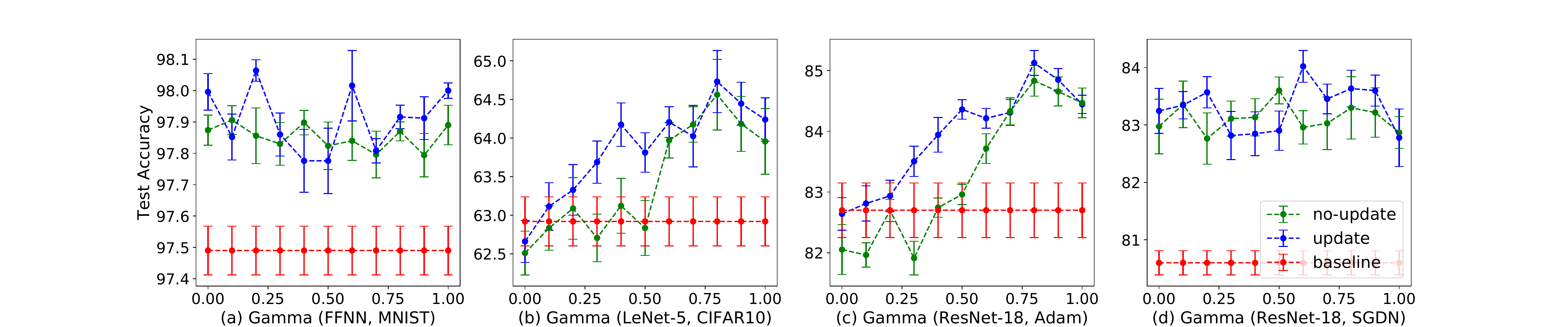}
 \caption{The diagram of model performances trained by Adam/SGDN-CAM-HD with different combination ratios in the case of two-level learning rates adaptation. The x-axis is the ratio of global-level adaptive learning rates. ResNet-18s are trained for 10 epochs only.} \label{Fig:1}
\end{center} 
\end{figure*}
which leads to the following findings: First, usually the optimal performance is neither at full global level nor full layer/filter level, but a weighted combination of two levels of adaptive learning rates, for both update and no-update cases. Second, CAM-HD methods outperform baseline Adam/SGDN methods for most of the combination ratios initializations. Third, updating of combination weights is effective and helpful in achieving better performance than applying fixed combination weights. 
This supports our analysis in Section~\ref{Sec:5.3.3}. Also, in real training processes, it is possible that the learning in favor of different combination weights in various stages and this requires the online adaptation of the combination weights.

\subsection{Feed Forward Neural Network for Image Classification}\label{Sec:5.4.3}
This experiment is conducted with feed-forward neural networks for image classification on MNIST, including 60,000 training examples and 10,000 test examples. We use the full training set for training and the full test set for validation. Three FFNN with three different hidden layer configurations are implemented \citep{svozil1997introduction, fine2006feedforward}, including [100, 100], [1000, 100], and [1000, 1000]. Adaptive optimizers including Adam, Adabound, Adam-HD with two hyper-gradient updating rates, and proposed Adam-CAM-HD are applied. For Adam-CAM-HD, we apply three-level parameter-layer-global adaptation with initialization of $\gamma_1=\gamma_2=0.3$ and $\gamma_3=0.4$, and two-level layer-global adaptation with $\gamma_1=\gamma_2=0.5$. No decay function of learning rates is applied.
\begin{figure*}[t] 
\begin{center}
 \includegraphics[width=1.0\linewidth]{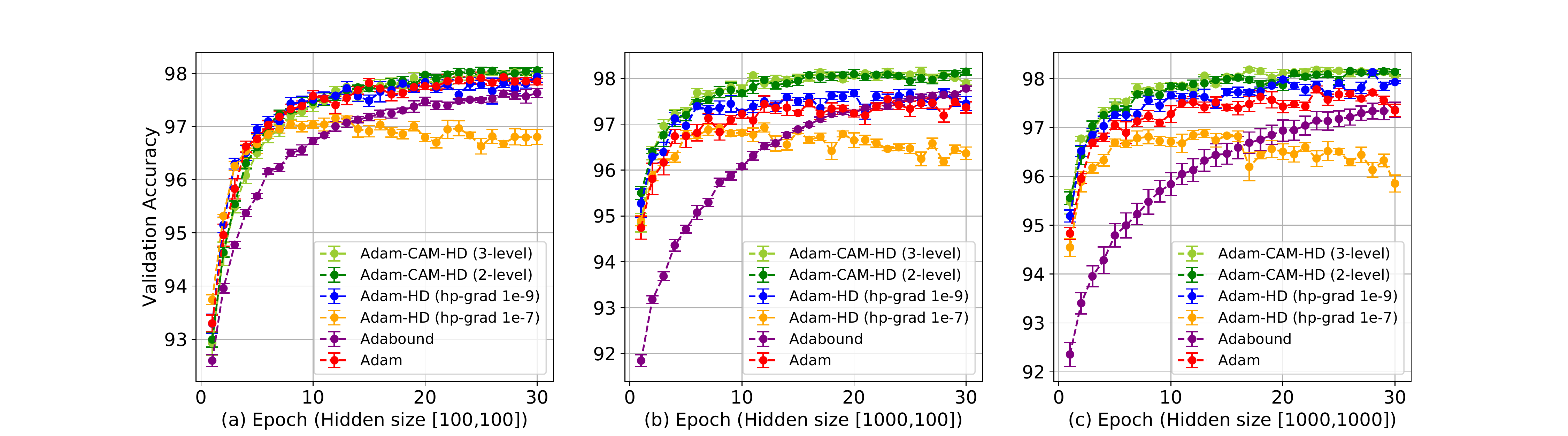}
 \caption{The comparison of learning curves of FFNN on MNIST with different adaptive optimizers.} \label{Fig:2}
\end{center} 
\end{figure*}
Figure~\ref{Fig:2} shows the validation accuracy curves for different optimizers during the training process of 30 epochs. We can learn that both the two-level and three-level Adam-CAM-HD outperform the baseline Adam optimizer with optimized hyper-parameters significantly. For Adam-HD, we find that the default hyper-gradient updating rate ($\beta=10^{-7}$) for Adam applied in \citep{baydin2017online} is not optimal in our experiments, while an optimized one of $10^{-9}$ can outperform Adam but still worse than Adam-CAM-HD with $\beta=10^{-7}$. 

The test accuracy of each setting and the corresponding standard error of the sample mean in 10 trials are given in Table~\ref{tab:3}.
\begin{table}[htbp]
  \centering{\footnotesize
  \caption{Summary of test performances with FFNNs.}
    \begin{tabular}{ccccccc}
    \toprule
          & \multicolumn{2}{c}{FFNN(100, 100)} & \multicolumn{2}{c}{FFNN(1000, 100)} & \multicolumn{2}{c}{FFNN(1000, 1000)} \\
\cmidrule{2-7}          & Test acc & Test S.E & Test acc & Test S.E & Test acc & Test S.E \\
    \midrule
    Adam-CAM-HD (3-level) & 97.91 & 0.07  & 97.92 & 0.15  & 98.29 & 0.07 \\
    Adam-CAM-HD (2-level) & \textbf{98.12} & 0.06  & \textbf{98.09} & 0.06  & \textbf{98.39} & 0.04 \\
    Adam-HD (hp-grad 1e-9) & 97.86 & 0.07  & 97.19 & 0.26  & 97.83 & 0.12 \\
    Adam  & 97.93 & 0.09  & 97.48 & 0.14  & 97.49 & 0.11 \\
    \bottomrule
    \end{tabular}
  \label{tab:3}}
\end{table}

\subsection{Lenet-5 for Image Classification}\label{Sec:5.4.4}
The second experiment is done with LeNet-5, a classical convolutional neural network without involving many building and training tricks \citep{lecun2015lenet}. We compare a set of adaptive Adam optimizers including Adam, Adam-HD, Adam-CAM-HD, Adabound, RAdam and L4 for the image classification learning task of MNIST, CIFAR10 and SVHN \citep{netzer2011reading}. For Adam-CAM-HD, we apply a two-level setting with filter-wise and global learning rates adaptation and initialize $\gamma_1=0.2$, $\gamma_2=0.8$. We also implement an exponential decay function $\tau(t)=\exp(-r t)$ as was discussed in Section~\ref{Sec:5.3.5} with rate $r=0.002$ for all the three datasets, while $t$ is the number of iterations. For L4, we implement the recommended L4 learning rate of 0.15. For Adabound and RAdam, we also apply the recommended hyper-parameters in the original papers. The other hyper-parameter settings are optimized in Section~\ref{Sec:5.4.1}.
\begin{figure*}[th] 
\begin{center}
 \includegraphics[width=1.0\linewidth]{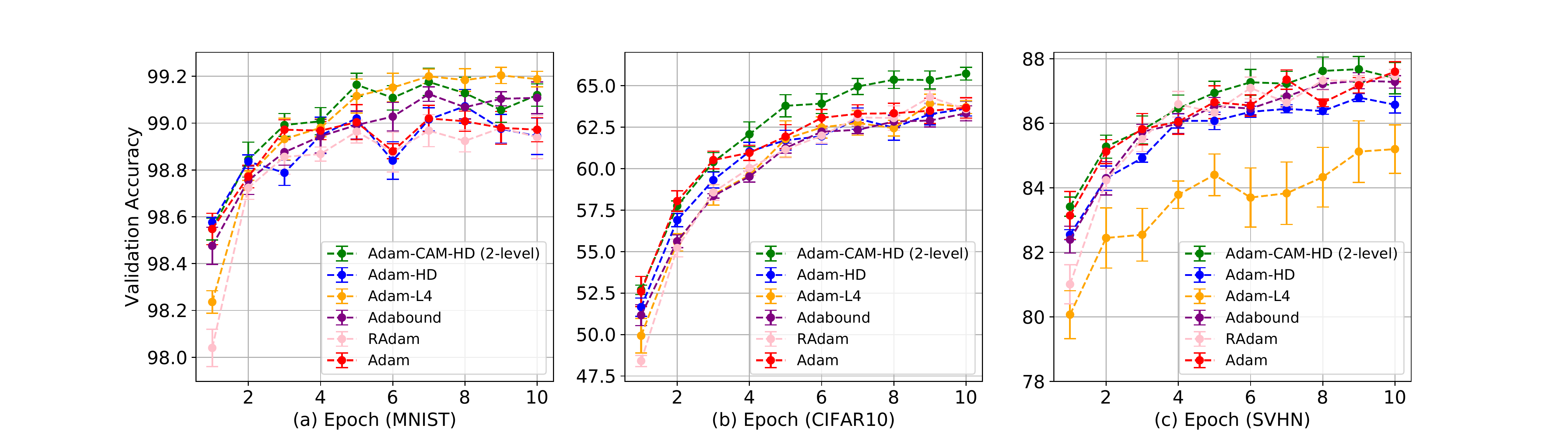}
 \caption{The comparison of learning curves of training LeNet-5 with different adaptive optimizers.} \label{Fig:3}
\end{center} 
\end{figure*}
As we can see in Figure~\ref{Fig:3}, Adam-CAM-HD again shows the advantage over other methods in all the three sub-experiments, except MNIST L4 that could perform better in a later stage. The experiment on SVHN indicates that the recommended hyper-parameters for L4 could fail in some cases with unstable accuracy curves. RAdam and Adabound outperform baseline Adam method on MNIST, while Adam-HD does not show a significant advantage over Adam with optimized hyper-gradient updating rate that is shared with Adam-CAM-HD.
The corresponding summary of test performance is given in Table~\ref{tab:4}, in which the test accuracy of Adam-CAM-HD outperform other optimizers on both CIFAR10 and SVHN. Especially, it gives significantly better results than Adam and Adam-HD for all the three datasets.
\begin{table}[htbp]
  \centering{\footnotesize
  \caption{Summary of test performances with LeNet-5}
    \begin{tabular}{ccccccc}
    \toprule
          & \multicolumn{2}{c}{MNIST} & \multicolumn{2}{c}{CIFAR10} & \multicolumn{2}{c}{SVHN} \\
\cmidrule{2-7}          & Test acc & Test S.E & Test acc & Test S.E & Test acc & Test S.E \\
\cmidrule{2-7}    Adam-CAM-HD & 98.93 & 0.07  & \textbf{65.55} & 0.18  & \textbf{87.58} & 0.37 \\
    Adam-HD & 98.83 & 0.05  & 63.3  & 0.66  & 86.94 & 0.13 \\
    Adam-L4 & \textbf{99.19} & 0.05  & 63.76 & 0.26  & 85.44 & 0.42 \\
    Adabound & 99.11 & 0.05  & 64.06 & 0.36  & 87.22 & 0.14 \\
    RAdam & 98.94 & 0.06  & 63.91 & 0.34  & 87.31 & 0.41 \\
    Adam  & 98.89 & 0.05  & 63.88 & 0.45  & 86.82 & 0.16 \\
    \bottomrule
    \end{tabular}
  \label{tab:4}}
\end{table}

\subsection{ResNet for Image Classification}\label{Sec:5.4.5}
In the third experiment, we apply ResNets for image classification task on CIFAR10 \citep{he2016deep, devries2017improved} following the code provided by \url{github.com/kuangliu/pytorch-cifar}. We compare Adam and Adam-based adaptive optimizers, as well as SGD with Nestorov momentum (SGDN) and corresponding adaptive optimizers for training both ResNet-18 and ResNet-34. For SGDN methods, we apply a learning rate schedule, in which the learning rate is initialized to a default value of 0.1 and reduced to 0.01 or 10\% (for SGDN-CAM-HD) after epoch 150. The momentum is set to be 0.9 for all SGDN methods. For Adam-CAM-HD SGDN-CAM-HD, we apply two-level CAM-HD with the same setting as the second experiment. We also implement Adabound-CAM-HD discussed in Section~\ref{Sec:5.3.4} by sharing the common parameters with Adabound. In addition, we apply an exponential decay function with a decay rate $r=0.001$ for all the CAM-HD methods. The learning curves for validation accuracy, training loss, and validation loss of ResNet-18 and ResNet-34 are shown in Figure~\ref{Fig:4}.
\begin{figure*}[th] 
\begin{center}
 \includegraphics[width=1.0\linewidth]{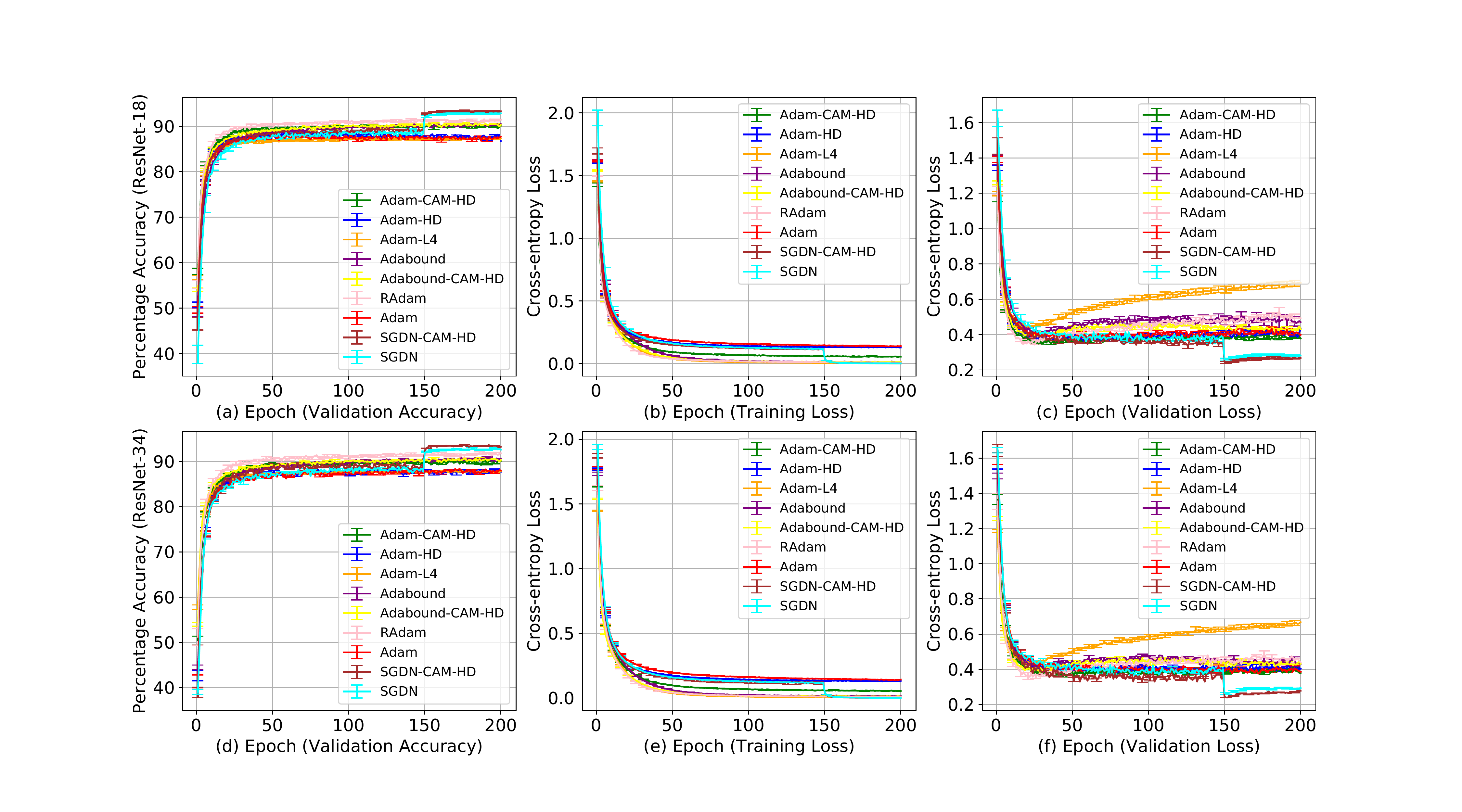}
 \caption{The learning curves of training ResNet-18/34 on CIFAR10 with adaptive optimizers.}\label{Fig:4}
\end{center} 
\end{figure*}
We can see that the validation accuracy of Adam-CAM-HD reaches about 90\% in 40 epochs and consistently outperforms Adam, L4 and Adam-HD optimizers in a later stage. The L4 optimizer with recommended hyper-parameter and an optimized weight-decay rate of 0.0005 (instead of 1e-4 applied in other Adam-based optimizers) can outperform baseline Adam for both ResNet-18 and ResNet-34, while its training loss outperforms all other methods but with potential over-fitting. Adam-HD achieves a similar or better validation accuracy than Adam with an optimized hyper-gradient updating rate of $10^{-9}$. RAdam performs slightly better than Adam-CAM-HD in terms of validation accuracy, but the validation cross-entropy of both RAdam and Adabound are outperformed by our method. Also, we find that in training ResNet-18/34, the validation accuracy and validation loss of SGDN-CAM-HD slightly outperform SGDN in most epochs even after the resetting of the learning rate at epoch 150. 
The test performances (average accuracy and standard error) of different optimizers for ResNet-18 and ResNet-34 after 200 epoch of training are shown in Table~\ref{tab:5}\footnote{Here Adam-based methods achieve much lower test accuracies as we only apply learning rate schedules to SGDN and SGDN-CAM-HD.}. 
We can learn that for both ResNet-18 and ResNet-34, the proposed CAM-HD methods (Adam-CAM-HD, Adabound-CAM-HD and SGDN-CAM-HD) can improve the corresponding baseline methods (Adam, Adabound and SGDN) with statistical significance. Especially, Adabound-CAM-HD outperforms both Adam-CAM-HD and Adabound.
\begin{table}[htbp]
  \centering{\small
  \caption{Summary of test performances with ResNet-18/34}
    \begin{tabular}{ccc}
    \toprule
      \textbf{Method}    & \textbf{ResNet-18} & \textbf{ResNet-34} \\
    \midrule
    \textbf{Adam}  & 87.03 (0.15) & 87.95 (0.22) \\
    \textbf{Adam-HD} & 87.26 (0.35) & 88.48 (0.48) \\
    \textbf{Adam-CAM-HD} & \textbf{90.31 (0.25)} & \textbf{90.28 (0.09)} \\
    \midrule
    \textbf{Adabound} & 90.29 (0.15) & 90.15 (0.30) \\
    \textbf{Adabound-CAM-HD} & \textbf{90.49 (0.31)} & \textbf{91.12 (0.23)} \\
    \midrule
    \textbf{SGDN}  & 93.04 (0.21) & 92.93 (0.29) \\
    \textbf{SGDN-CAM-HD} & \textbf{93.35 (0.08)} & \textbf{93.47 (0.23)} \\
    \bottomrule
    \end{tabular}%
  \label{tab:5}}%
\end{table}%



\section{Discussion} \label{Sec:5.5}

The experiments on both small models and large models demonstrate the advantage of the proposed method over baseline optimizers in terms of validation and test accuracy. One explanation of the performance improvement of our method is that it achieves a higher level of adaptation by introducing hierarchical learning rate structures with learn-able combination weights, while the over-parameterization of adaptive learning rates is controlled by its intrinsic regularization effects. In addition, experiments show that the performance improvement does not require tuning the hyper-parameters independently if the task or model is similar. For example, the hyper-gradient updating rate for LeNet-5, ResNet-18 and ResNet-34 are all set to be 1e-8 in our experiments no matter the dataset being learned. Also, the hyper-parameter CAM-HD-lr is shared among each group of models (FFNNs, LeNet-5, ResNets) for all datasets being learned. 
For the combination ratio, $\gamma_1=0.2$, $\gamma_2=0.8$ works for all our experiments with convolutional networks. However, as the loss surface with respect to the combination weights may not be convex for deep learning models, the learning of combination weights may fall into local optimal. Therefore, it is possible that several trials are needed to find a good initialization of combination weights although the learning of combination weights works locally \citep{feurer2019hyperparameter}. In general, the selected hyper-parameters are transferable to a similar task for an improvement from the corresponding baseline, while the optimal hyper-parameter setting may shift a bit. 

The proposed CAM-HD method can also apply learning rate schedules in many ways to achieve further improvement. One example is our ResNet experiment on CIFAR10 with SGDN and SGDN-CAM-HD. For more advanced learning rate schedules \citep{lang2019using, ge2019step}, we can apply strategies like piece-wise adaptive scheme by re-initialize all the levels for different steps. Another method is to replace global level learning rate with scheduled learning rate, while adapting the combination weights and other levels continuously.

\subsection{Learning of combination weights}\label{Sec:5.5.1}
The following figures including Figure~\ref{Fig:5}, Figure~\ref{Fig:6}, Figure~\ref{Fig:7} and Figure~\ref{Fig:8} give the learning curves of combination weights with respect to the number of training iterations in each experiments, in which each curve is averaged by 5 trials with error bars. Through these figures, we can compare the updating curves with different models, different datasets and different CAM-HD optimizers. 
\begin{figure}[th] 
\begin{center}
 \includegraphics[width=1.0\linewidth]{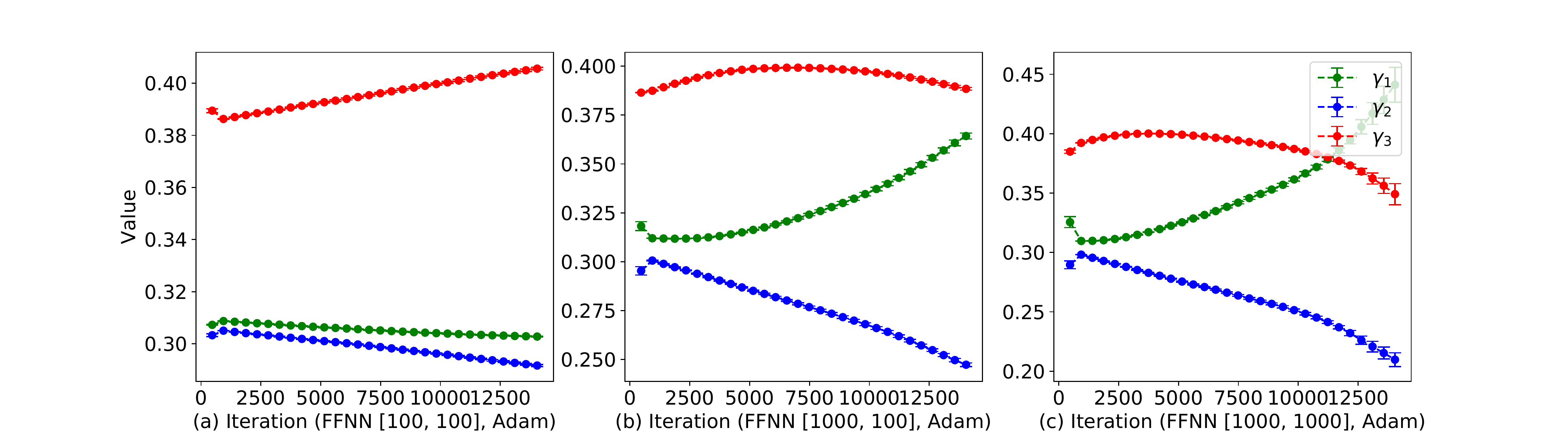}
 \caption{Learning curves of $\gamma$s for FFNN on MNIST with Adam.} \label{Fig:5} 
\end{center} 
\end{figure}
\begin{figure}[th] 
\begin{center}
 \includegraphics[width=1.0\linewidth]{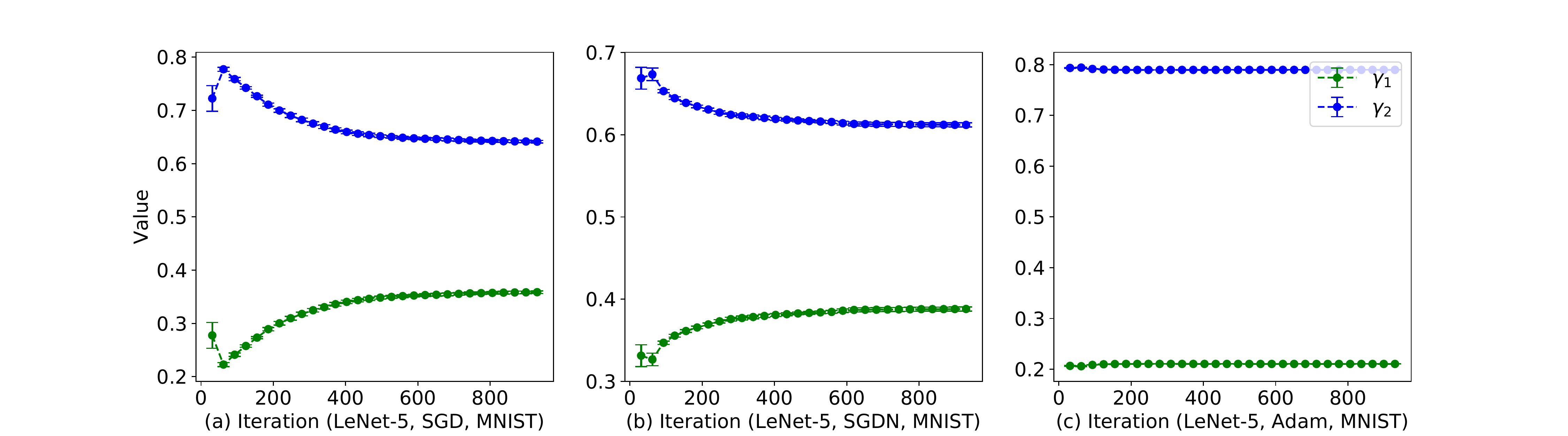}
 \caption{Learning curves of $\gamma$s for LeNet-5 on MNIST with SGD, SGDN and Adam ($\tau = 0.002$).} \label{Fig:6} 
\end{center} 
\end{figure}
\begin{figure}[th] 
\begin{center}
 \includegraphics[width=0.7\linewidth]{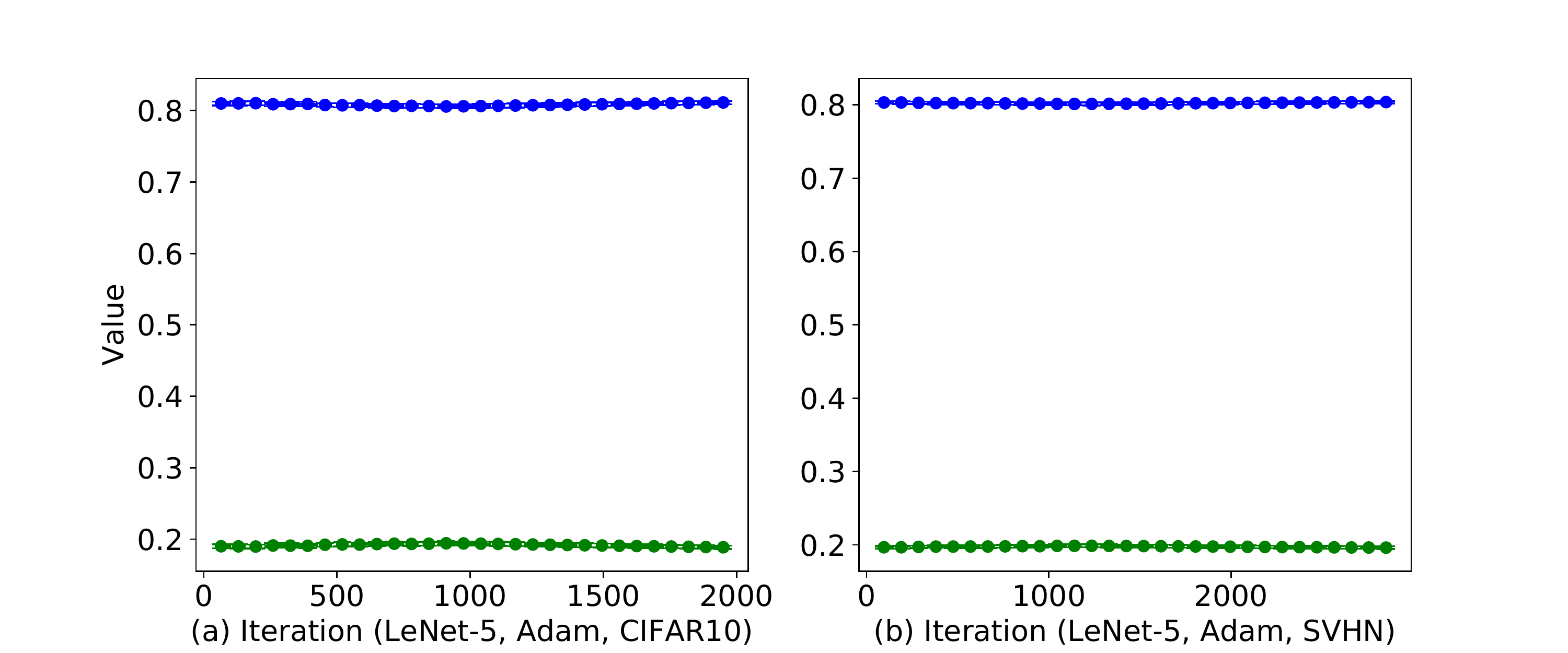}
 \caption{Learning curves of $\gamma$s for LeNet-5 with Adam-CAM-HD on CIFAR10 and SVHN ($\tau = 0.002$).} \label{Fig:7} 
\end{center} 
\end{figure}
\begin{figure}[th] 
\begin{center}
 \includegraphics[width=0.7\linewidth]{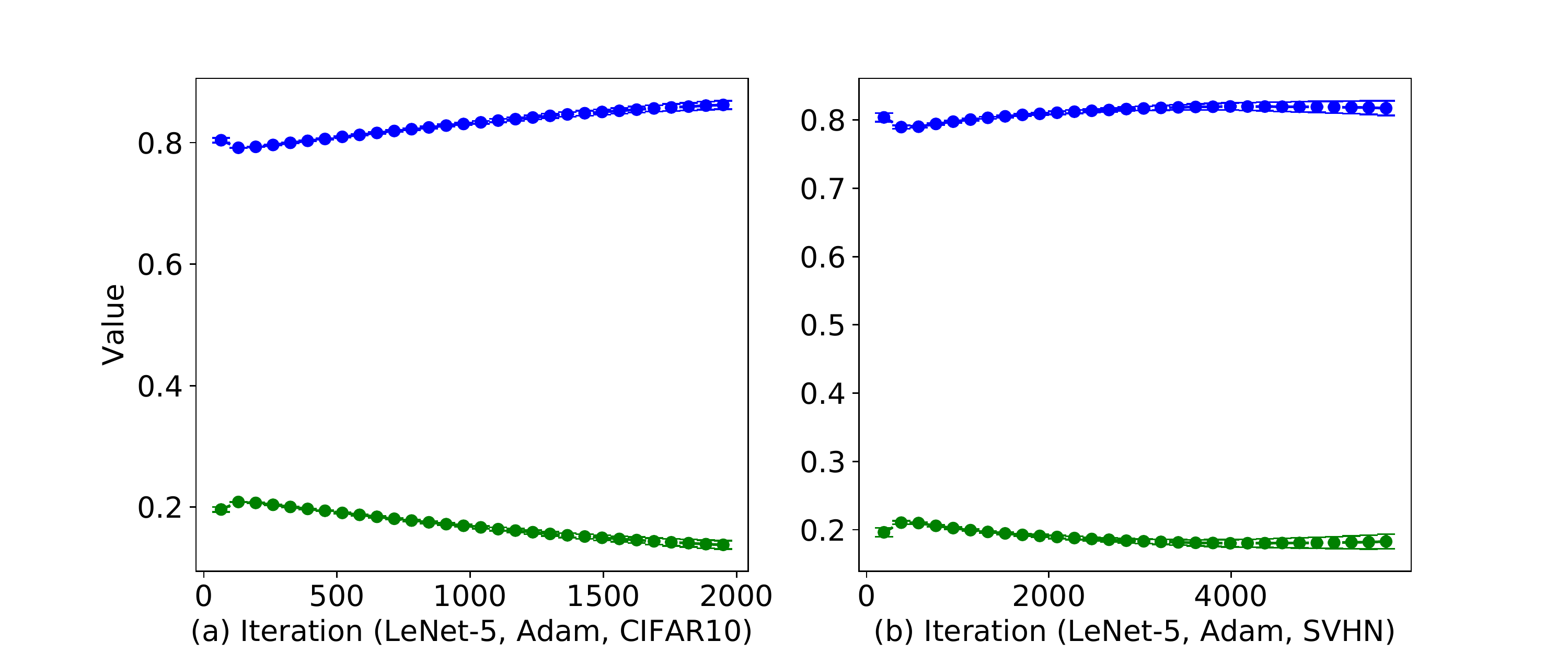}
 \caption{Learning curves of $\gamma$s for LeNet-5 with Adam-CAM-HD on CIFAR10 and SVHN ($\tau = 0.001$).} \label{Fig:7b} 
\end{center} 
\end{figure}
\begin{figure}[th] 
\begin{center}
 \includegraphics[width=0.7\linewidth]{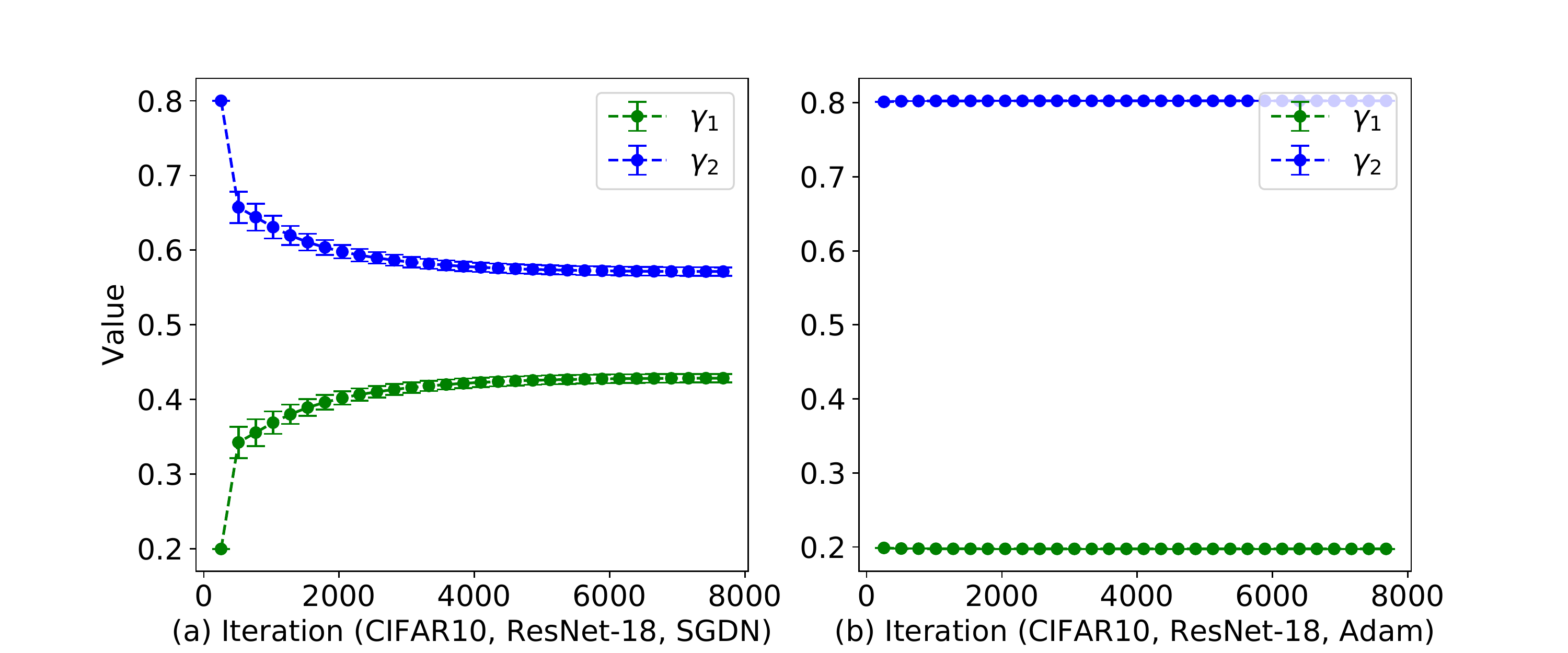}
 \caption{Learning curves of $\gamma$s for ResNet-18 with SGDN-CAM-HD and Adam-CAM-HD ($\tau = 0.001$).} \label{Fig:8} 
\end{center} 
\end{figure}
Figure~\ref{Fig:5} corresponds to the experiment of FFNN on MNIST in Section 3.3 of the main paper, which is a three-level case. We can see that for different FFNN architecture, the learning behaviors of $\gamma$s also show different patterns, although trained on a same dataset. Meanwhile, the standard errors for multiple trials are much smaller relative to the changes of the average combination weight values. 

Figure~\ref{Fig:6} corresponds to the learning curves of $\gamma$s in the experiments of LeNet-5 for MNIST image classification with SGD, SGDN and Adam, which are trained on 10\% of original training dataset. In addition, Figure~\ref{Fig:7} corresponds to the learning curves of $\gamma$s in the experiments of LeNet-5 for CIFAR10 and SVHN image classification with Adam-CAM-HD.

As is shown in Figure~\ref{Fig:6}, for SGD-CAM-HD, SGDN-CAM-HD and Adam-CAM-HD, the equilibrium values of combination weights are different from each other. Although the initialization $\gamma_1=0.2$, $\gamma_2=0.8$ and the updating rate $\delta=0.03$ are set to be the same for the three optimizers, the values of $\gamma_1$ and $\gamma_2$ only change in a small proportion when training with Adam-CAM-HD, while the change is much more significant towards larger filter/layer-wise adaptation when SGD-CAM-HD or SGDN-CAM-HD is implemented. The numerical results show that for SGDN-CAM-HD, the average value of weight for layer-wise adaptation $\gamma_1$ jumps from 0.2 to 0.336 in the first epoch, then drop back to 0.324 before keeping increasing till about 0.388. For Adam-CAM-HD, the average $\gamma_1$ moves from 0.20 to 0.211 with about 5\% change. In Figure~\ref{Fig:7}, both the two subplots are about LeNet-5 models trained with Adam-CAM-HD, while the exponential decay rate for weighted approximation is set to be $\tau = 0.002$. For the updating curves in Figure~\ref{Fig:7}(a), which is trained on CIFAR10 with Adam-CAM-HD, the combination weight for filter-wise adaptation moves from 0.20 to 0.188. Meanwhile, for the updating curves in Figure~\ref{Fig:7}(b), which is trained on SVHN, the combination weight for filter-wise adaptation moves from 0.20 to 0.195. Further exploration shows that $\tau$ has an impact on the learning curves of combination weights. As is shown by Figure~\ref{Fig:7b}, a smaller $\tau=0.001$ can result in a more significant change of combination weights during training with Adam-CAM-HD.   
The similar effect can also be observed from the learning curves of $\gamma$s for ResNet-18, which is given in Figure~\ref{Fig:8} and we only take the first 8,000 iterations. Again, we find that in training ResNet-18 on CIFAR10, the combination weights of SGD/SGDN-CAM-HD change much faster than that of Adam-CAM-HD. There are several reasons for this effect: First, in the cases when $\gamma$s do not move significantly, we apply Adam-CAM-HD, where the main learning rate (1e-3) is only about 1\%-6\% of the learning rate of SGD or SGDN (1e-1). In Algorithm 1, we can see that the updating rate of $\gamma$s is in proportion of alpha given other terms unchanged. Thus, for the same tasks, if the same value of updating rate $\delta$ is applied, the updating scale of $\gamma$s for Adam-CAM-HD can be much smaller than that for SGDN-CAM-HD. Second, this does not mean that if we apply a much larger $\delta$ for Adam-CAM-HD, the combination weights will still not change significantly or the performance will not be improved. It simply means that using a small $\delta$ can also achieve good performance due to the goodness of initialisation points. Third, it is possible that Adam requires lower level of combination ratio adaptation for the same network architecture compared with SGD/SGDN due to the fact that Adam itself involves stronger adaptiveness.

\subsection{Number of parameters and space complexity}\label{Sec:5.5.2}
The proposed adaptive optimizer is for efficiently updating the model parameters, while the final model parameters will not be increase by introducing CMA-HD optimizer. However, during the training process, several extra intermediate variables are introduced. For example, in the discussed three-level's case for feed-forward neural network with $n_{\text{layer}}$ 
layers, we need to restore $h_{p,t}$, $h_{l,t}$ and $h_{g,t}$, which have the sizes of $S(h_{p,t}) = \sum^{n_{\text{layer}}-1}_{l=1} (n_l+1) n_{l+1}$, $S(h_{l,t}) = n_{\text{layer}}$ and $S(h_{g,t}) = 1$, respectively, 
where $n_i$ is the number of units in $i$th layer. Also, learning rates $\alpha_{p,t}$, $\alpha_{l,t}$, $\alpha_{g,t}$ and take the sizes of $S(a_{p,t}) = \sum^{n_{\text{layer}}-1}_{l=1} (n_l+1) n_{l+1}$, $S(a_{l,t}) = n_{\text{layer}}$, $S(a_{g,t}) = 1$, $S(a_{g,t}) = 1$, and $S(a^{*}_{p,t}) = \sum^{n_{\text{layer}}-1}_{l=1} (n_l+1) n_{l+1}$, respectively. Also we need a small set of scalar parameters to restore $\gamma_1$, $\gamma_2$ and $\gamma_3$ and other coefficients.
\newline\newline
Consider the fact that the training the baseline models, we need to restore model parameters, corresponding gradients, as well as the intermediate gradients during the implementation of chain rule, CAM-HD will take twice of the space for storing intermediate variables in the worst case. For two-level learning rate adaptation considering global and layer-wise learning rates, the extra space complexity by CAM-HD will be one to two orders' smaller than that of baseline model during training.

\subsection{Time Complexity}\label{Sec:5.5.3}
In CMA-HD, we need to calculate gradient of loss with respect to the learning rates in each level, which are $h_{p, t}$, $h_{l, t}$ and $h_{g,t}$ in three-level's case. However, the gradient of each parameter is already known during normal model training, the extra computational cost comes from taking summations and updating the lowest-level learning rates. In general, this cost is in linear relation with the number of differentiable parameters in the original models. Here we discuss the case of feed-forward networks and convolutional networks.
\newline\newline
Recall that for feed-forward neural network the whole computational complexity is:
\begin{equation}
    T(n) = O(m\cdot n_{\text{iter}}\cdot\sum_{l=2}^{n_{\text{layer}}}n_l\cdot n_{l-1}\cdot n_{l-2})
\end{equation}
where $m$ is the number of training examples, $n_{\text{iter}}$ is the iterations of training, $n_l$ is the number of units in the $l$-th layer. On the other hand, when using three-level CAM-HD with, where the lowest level is parameter-wise, we need $n_{\text{layer}}$ element products to calculate $h_{p,t}$ for all layers, one $n_{\text{layer}}$ matrix element summations to calculate $h_{l,t}$ for all layers, as well as a list summation to calculate $h_{g,t}$. In addition, two element-wise summations will also be implemented for calculating $\alpha_{p,t}$ and $\alpha^{*}_p$. Therefore, the extra computational cost of using CAM-HD is $\Delta T(n) = O(m_b\cdot n_{\text{iter}}\sum^{n_{\text{layer}}}_{l=2} (n_l\cdot n_{l-1}+n_l))$, where $m_b$ is the number of mini-batches for training. Notice that $m/m_b$ is the batch size, which is usually larger than 100. This extra cost is more than one-order smaller than the computation complexity of training a model without learning rate adaptation. For the cases when the lowest level is layer-wise, only one element-wise matrix product is needed in each layer to calculate $h_{l,t}$. For convolutional neural networks, we have learned that the total time complexity of all convolutional layers is \citep{he2015convolutional}:
\begin{equation}
O (m \cdot n_{\text{\text{iter}}} \cdot \sum^{n_{conv\_layer}}_{l=1} (n_{l-1}\cdot s^2_l\cdot n_l \cdot m^2_l)) 
\end{equation}
where $l$ is the index of a convolutional layer, and $n_{conv\_layer}$ is the
depth (number of convolutional layers). $n_l$ is the number of filters in the $l$-th layer, while $n_{l-1}$ is known as the number of input channels of the $l$-th layer. $s_l$ is the spatial size of the filter. $m_l$ is the spatial size of the output feature map. If we consider convolutional filters as layers, the extra computational cost for CAM-HD in this case is $\Delta T(n) = O(m_b\cdot n_{\text{iter}}\sum^{n_{conv\_layer}}_{l=1} ((n_{l-1}\cdot s^2_l+1)\cdot n_l))$, which is still more than one order smaller than the cost of model without learning rate adaptation.
\newline\newline
Therefore, for large networks, applying CMA-HD will not significantly increase the computational cost from the theoretical prospective.

\section{Conclusion} \label{Sec:5.6}

In this study, we propose a gradient-based learning rate adaptation strategy by introducing hierarchical learning rate structures in deep neural networks. By considering the relationship between regularization and the combination of adaptive learning rates in multiple levels, we further propose a joint algorithm for adaptively learning each level's combination weight. It increases the adaptiveness of the hyper-gradient descent method in any single level, while over-parameterization involved in optimizers can be controlled by adaptive regularization effect. Experiments on FFNN, LeNet-5, and ResNet-18/34 indicate that the proposed methods can outperform the standard ADAM/SGDN and other baseline methods with statistical significance. 



\newpage

\bibliography{MAIN_arxiv}

\begin{thebibliography}{36}
\providecommand{\natexlab}[1]{#1}
\providecommand{\url}[1]{\texttt{#1}}
\expandafter\ifx\csname urlstyle\endcsname\relax
  \providecommand{\doi}[1]{doi: #1}\else
  \providecommand{\doi}{doi: \begingroup \urlstyle{rm}\Url}\fi

\bibitem[Almeida et~al.(1998)Almeida, Langlois, Amaral, and
  Plakhov]{almeida1998parameter}
L.~B. Almeida, T.~Langlois, J.~D. Amaral, and A.~Plakhov.
\newblock Parameter adaptation in stochastic optimization.
\newblock \emph{On-Line Learning in Neural Networks, Publications of the Newton
  Institute}, pages 111--134, 1998.

\bibitem[Andrychowicz et~al.(2016)Andrychowicz, Denil, Gomez, Hoffman, Pfau,
  Schaul, Shillingford, and De~Freitas]{andrychowicz2016learning}
M.~Andrychowicz, M.~Denil, S.~Gomez, M.~W. Hoffman, D.~Pfau, T.~Schaul,
  B.~Shillingford, and N.~De~Freitas.
\newblock Learning to learn by gradient descent by gradient descent.
\newblock In \emph{NeurIPS}, pages 3981--3989, 2016.

\bibitem[Baydin et~al.(2017)Baydin, Cornish, Rubio, Schmidt, and
  Wood]{baydin2017online}
A.~G. Baydin, R.~Cornish, D.~M. Rubio, M.~Schmidt, and F.~Wood.
\newblock Online learning rate adaptation with hypergradient descent.
\newblock \emph{ICLR}, 2017.

\bibitem[Baydin et~al.(2018)Baydin, Pearlmutter, Radul, and
  Siskind]{baydin2018automatic}
A.~G. Baydin, B.~A. Pearlmutter, A.~A. Radul, and J.~M. Siskind.
\newblock Automatic differentiation in machine learning: a survey.
\newblock \emph{JMLR}, 18\penalty0 (153), 2018.

\bibitem[Bergstra et~al.(2011)Bergstra, Bardenet, Bengio, and
  K{\'e}gl]{bergstra2011algorithms}
J.~Bergstra, R.~Bardenet, Y.~Bengio, and B.~K{\'e}gl.
\newblock Algorithms for hyper-parameter optimization.
\newblock In \emph{NeurIPS)}, volume~24. Neural Information Processing Systems
  Foundation, 2011.

\bibitem[DeVries and Taylor(2017)]{devries2017improved}
T.~DeVries and G.~W. Taylor.
\newblock Improved regularization of convolutional neural networks with cutout.
\newblock \emph{arXiv preprint arXiv:1708.04552}, 2017.

\bibitem[Duchi et~al.(2011)Duchi, Hazan, and Singer]{duchi2011adaptive}
J.~Duchi, E.~Hazan, and Y.~Singer.
\newblock Adaptive subgradient methods for online learning and stochastic
  optimization.
\newblock \emph{JMLR}, 12:\penalty0 2121--2159, 2011.

\bibitem[Feurer and Hutter(2019)]{feurer2019hyperparameter}
M.~Feurer and F.~Hutter.
\newblock Hyperparameter optimization.
\newblock In \emph{Automated Machine Learning}, pages 3--33. Springer, Cham,
  2019.

\bibitem[Fine(2006)]{fine2006feedforward}
T.~L. Fine.
\newblock \emph{Feedforward neural network methodology}.
\newblock Springer Science \& Business Media, 2006.

\bibitem[Franceschi et~al.(2017)Franceschi, Donini, Frasconi, and
  Pontil]{franceschi2017forward}
L.~Franceschi, M.~Donini, P.~Frasconi, and M.~Pontil.
\newblock Forward and reverse gradient-based hyperparameter optimization.
\newblock In \emph{ICML}, pages 1165--1173. JMLR. org, 2017.

\bibitem[Fu et~al.(2017)Fu, Ng, Chen, Ilievski, Pal, and Chua]{fu2017neural}
J.~Fu, R.~Ng, D.~Chen, I.~Ilievski, C.~Pal, and T.-S. Chua.
\newblock Neural optimizers with hypergradients for tuning parameter-wise
  learning rates.
\newblock \emph{JMLR: Workshop and Conference Proceedings}, 1:\penalty0 1--8,
  2017.

\bibitem[Ge et~al.(2019)Ge, Kakade, Kidambi, and Netrapalli]{ge2019step}
R.~Ge, S.~M. Kakade, R.~Kidambi, and P.~Netrapalli.
\newblock The step decay schedule: A near optimal, geometrically decaying
  learning rate procedure for least squares.
\newblock In \emph{Advances in Neural Information Processing Systems}, pages
  14977--14988, 2019.

\bibitem[He and Sun(2015)]{he2015convolutional}
K.~He and J.~Sun.
\newblock Convolutional neural networks at constrained time cost.
\newblock In \emph{CVPR}, pages 5353--5360, 2015.

\bibitem[He et~al.(2016)He, Zhang, Ren, and Sun]{he2016deep}
K.~He, X.~Zhang, S.~Ren, and J.~Sun.
\newblock Deep residual learning for image recognition.
\newblock In \emph{CVPR}, pages 770--778, 2016.

\bibitem[Karimi et~al.(2016)Karimi, Nutini, and Schmidt]{karimi2016linear}
H.~Karimi, J.~Nutini, and M.~Schmidt.
\newblock Linear convergence of gradient and proximal-gradient methods under
  the polyak-{\l}ojasiewicz condition.
\newblock In \emph{Joint European Conference on Machine Learning and Knowledge
  Discovery in Databases}, pages 795--811. Springer, 2016.

\bibitem[Kingma and Ba(2015)]{kingma2014adam}
D.~P. Kingma and J.~Ba.
\newblock Adam: A method for stochastic optimization.
\newblock \emph{ICLR}, 2015.

\bibitem[Krizhevsky and Hinton(2012)]{krizhevsky2009learning}
A.~Krizhevsky and G.~Hinton.
\newblock Learning multiple layers of features from tiny images.
\newblock \emph{University of Toronto}, 2012.

\bibitem[Lang et~al.(2019)Lang, Xiao, and Zhang]{lang2019using}
H.~Lang, L.~Xiao, and P.~Zhang.
\newblock Using statistics to automate stochastic optimization.
\newblock In \emph{Advances in Neural Information Processing Systems}, pages
  9540--9550, 2019.

\bibitem[LeCun et~al.(1998)LeCun, Bottou, Bengio, and
  Haffner]{lecun1998gradient}
Y.~LeCun, L.~Bottou, Y.~Bengio, and P.~Haffner.
\newblock Gradient-based learning applied to document recognition.
\newblock \emph{Proceedings of the IEEE}, 86\penalty0 (11):\penalty0
  2278--2324, 1998.

\bibitem[LeCun et~al.(2015)]{lecun2015lenet}
Y.~LeCun et~al.
\newblock Lenet-5, convolutional neural networks.
\newblock \emph{URL: http://yann. lecun. com/exdb/lenet}, 20:\penalty0 5, 2015.

\bibitem[Liu et~al.(2019)Liu, Jiang, He, Chen, Liu, Gao, and
  Han]{liu2019variance}
L.~Liu, H.~Jiang, P.~He, W.~Chen, X.~Liu, J.~Gao, and J.~Han.
\newblock On the variance of the adaptive learning rate and beyond.
\newblock In \emph{ICLR}, 2019.

\bibitem[Luo et~al.(2018)Luo, Xiong, Liu, and Sun]{luo2019adaptive}
L.~Luo, Y.~Xiong, Y.~Liu, and X.~Sun.
\newblock Adaptive gradient methods with dynamic bound of learning rate.
\newblock In \emph{ICLR}, 2018.

\bibitem[Lv et~al.(2017)Lv, Jiang, and Li]{lv2017learning}
K.~Lv, S.~Jiang, and J.~Li.
\newblock Learning gradient descent: Better generalization and longer horizons.
\newblock In \emph{ICML}, pages 2247--2255. JMLR. org, 2017.

\bibitem[Maclaurin et~al.(2015)Maclaurin, Duvenaud, and
  Adams]{maclaurin2015gradient}
D.~Maclaurin, D.~Duvenaud, and R.~Adams.
\newblock Gradient-based hyperparameter optimization through reversible
  learning.
\newblock In \emph{ICML}, pages 2113--2122, 2015.

\bibitem[Netzer et~al.(2011)Netzer, Wang, Coates, Bissacco, Wu, and
  Ng]{netzer2011reading}
Y.~Netzer, T.~Wang, A.~Coates, A.~Bissacco, B.~Wu, and A.~Y. Ng.
\newblock Reading digits in natural images with unsupervised feature learning.
\newblock In \emph{NIPS Workshop on Deep Learning and Unsupervised Feature
  Learning 2011}, 2011.

\bibitem[Prechelt(1998)]{prechelt1998early}
L.~Prechelt.
\newblock Early stopping-but when?
\newblock In \emph{Neural Networks: Tricks of the trade}, pages 55--69.
  Springer, 1998.

\bibitem[Reddi et~al.(2019)Reddi, Kale, and Kumar]{reddi2019convergence}
S.~J. Reddi, S.~Kale, and S.~Kumar.
\newblock On the convergence of adam and beyond.
\newblock In \emph{International Conference on Learning Representations}, 2019.

\bibitem[Rolinek and Martius(2018)]{rolinek2018l4}
M.~Rolinek and G.~Martius.
\newblock L4: Practical loss-based stepsize adaptation for deep learning.
\newblock In \emph{NeurIPS}, pages 6433--6443, 2018.

\bibitem[Ruder(2016)]{ruder2016overview}
S.~Ruder.
\newblock An overview of gradient descent optimization algorithms.
\newblock \emph{arXiv:1609.04747}, 2016.

\bibitem[Savarese(2019)]{savarese2019convergence}
P.~Savarese.
\newblock On the convergence of adabound and its connection to sgd.
\newblock \emph{arXiv:1908.04457}, 2019.

\bibitem[Subramanian(2018)]{subramanian2018deep}
V.~Subramanian.
\newblock \emph{Deep Learning with PyTorch: A practical approach to building
  neural network models using PyTorch}.
\newblock Packt Publishing Ltd, 2018.

\bibitem[Sun(2019)]{sun2019optimization}
R.~Sun.
\newblock Optimization for deep learning: theory and algorithms.
\newblock \emph{arXiv:1912.08957}, 2019.

\bibitem[Svozil et~al.(1997)Svozil, Kvasnicka, and
  Pospichal]{svozil1997introduction}
D.~Svozil, V.~Kvasnicka, and J.~Pospichal.
\newblock Introduction to multi-layer feed-forward neural networks.
\newblock \emph{Chemometrics and intelligent laboratory systems}, 39\penalty0
  (1):\penalty0 43--62, 1997.

\bibitem[Tieleman and Hinton(2012)]{tieleman2012rmsprop}
T.~Tieleman and G.~Hinton.
\newblock Rmsprop: Divide the gradient by a running average of its recent
  magnitude. coursera: Neural networks for machine learning.
\newblock \emph{Tech. Rep., Technical report}, page~31, 2012.

\bibitem[Wichrowska et~al.(2017)Wichrowska, Maheswaranathan, Hoffman,
  Colmenarejo, Denil, de~Freitas, and Sohl-Dickstein]{wichrowska2017learned}
O.~Wichrowska, N.~Maheswaranathan, M.~W. Hoffman, S.~G. Colmenarejo, M.~Denil,
  N.~de~Freitas, and J.~Sohl-Dickstein.
\newblock Learned optimizers that scale and generalize.
\newblock In \emph{ICML}, pages 3751--3760. JMLR. org, 2017.

\bibitem[Zhang et~al.(2019)Zhang, Lucas, Ba, and Hinton]{zhang2019lookahead}
M.~Zhang, J.~Lucas, J.~Ba, and G.~E. Hinton.
\newblock Lookahead optimizer: k steps forward, 1 step back.
\newblock In \emph{NeurIPS}, pages 9593--9604, 2019.

\end{thebibliography}

\end{document}